\documentclass{article}
\usepackage[utf8]{inputenc}
\usepackage{natbib}
\usepackage{fancyhdr, fullpage, parskip,amsfonts}
\usepackage{times,latexsym}
\usepackage[T1]{fontenc}
\usepackage{algorithm,algorithmic}
\usepackage{colortbl, xcolor} 
\definecolor{SkyBlue}{RGB}{135, 206, 235}
\usepackage{siunitx}
\usepackage{tcolorbox}
\tcbuselibrary{breakable, listings}

\usepackage{fvextra}
\usepackage{wrapfig}
\makeatletter
\long\def\@makecaption#1#2{
  \vskip 0.8ex
  \setbox\@tempboxa\hbox{\small {\bf #1:} #2}
  \parindent 1.5em  
  \dimen0=\hsize
  \advance\dimen0 by -3em
  \ifdim \wd\@tempboxa >\dimen0
    \hbox to \hsize{
      \parindent 0em
      \hfil
      \parbox{\dimen0}{\def\baselinestretch{0.96}\small
        {\bf #1.} #2
      }
      \hfil}
  \else \hbox to \hsize{\hfil \box\@tempboxa \hfil}
  \fi
}

\makeatother
\usepackage{microtype}
\usepackage{graphicx}
\usepackage{subfigure}
\usepackage{booktabs} 
\usepackage{arydshln} 
\usepackage{tikz}
\usetikzlibrary{bayesnet} 
\usetikzlibrary{arrows}
\usepackage{amsmath}
\usepackage{amssymb}
\usepackage{mathtools}
\usepackage{amsthm}

\usepackage{enumitem,multirow,adjustbox}

\usepackage{hyperref}
\usepackage{url}
\usepackage{xcolor}		
\definecolor{darkblue}{rgb}{0, 0, 0.5}
\definecolor{beaublue}{rgb}{0.74, 0.83, 0.9}
\definecolor{gainsboro}{rgb}{0.86, 0.86, 0.86}
\definecolor{kleinblue}{rgb}{0,0.18,0.65}
\hypersetup{colorlinks=true,citecolor=kleinblue, linkcolor=kleinblue, urlcolor=kleinblue}

\usepackage{pifont}
\newtheorem{theorem}{Theorem}[section]
\newtheorem{proposition}[theorem]{Proposition}
\newtheorem{lemma}[theorem]{Lemma}

\newtheorem{remark}[theorem]{Remark}

\usepackage{amsmath,amsfonts,bm}

\def\eqref#1{equation~\ref{#1}}

\def\1{\bm{1}}

\DeclareMathAlphabet{\mathsfit}{\encodingdefault}{\sfdefault}{m}{sl}
\SetMathAlphabet{\mathsfit}{bold}{\encodingdefault}{\sfdefault}{bx}{n}

\def\gA{{\mathcal{A}}}

\def\gD{{\mathcal{D}}}

\def\gM{{\mathcal{M}}}

\usepackage{enumitem,algorithm,algorithmic}

\usepackage{xspace}

\makeatletter
\newcommand*\rel@kern[1]{\kern#1\dimexpr\macc@kerna}
\newcommand*\widebar[1]{%
  \begingroup
  \def\mathaccent##1##2{%
    \rel@kern{0.8}%
    \overline{\rel@kern{-0.8}\macc@nucleus\rel@kern{0.2}}%
    \rel@kern{-0.2}%
  }%
  \macc@depth\@ne
  \let\math@bgroup\@empty \let\math@egroup\macc@set@skewchar
  \mathsurround\z@ \frozen@everymath{\mathgroup\macc@group\relax}%
  \macc@set@skewchar\relax
  \let\mathaccentV\macc@nested@a
  \macc@nested@a\relax111{#1}%
  \endgroup
}
\makeatother

\usepackage[hyperpageref]{backref}

\renewcommand*{\backrefalt}[4]{
    \ifcase #1 \relax
    \or
        (Cited on page #2)
    \else
        (Cited on pages #2)
    \fi
}

\definecolor{Gray}{gray}{0.9}

\graphicspath{{./fig/}}

\usepackage{wrapfig}

\newcommand{\ours}[0]{\text{LeGIT}\xspace}

\newcommand{\oursfull}[0]{\textbf{L}arge Languag\textbf{e} Model \textbf{G}uided \textbf{I}ntervention \textbf{T}argeting\xspace}

\definecolor{purple}{HTML}{000000}

\newcommand{\customfootnotetext}[2]{{%
      \renewcommand{\thefootnote}{#1}%
      \footnotetext[0]{#2}}}%
\newcommand{\ourtitle}{Can Large Language Models Help Experimental Design for Causal Discovery?\xspace}
\begin{document}

\title{\ourtitle}

\author{
  Junyi Li$^{*\dag}$
  \and
  Yongqiang Chen$^{*\diamond\ddag}$
  \and
  Chenxi Liu$^\P$
  \and
  Qianyi Cai$^\dag$
  \and
  Tongliang Liu$^{\S\diamond}$
  \and
  Bo Han$^\P$
  \and
  Kun Zhang$^{\diamond \ddag}$
  \and
  Hui Xiong$^\dag$
}

\date{
  $^\dag$ The Hong Kong University of Science and Technology (Guangzhou)\\
  $^\diamond$ MBZUAI \quad  $\P$ Hong Kong Baptist University\\
  $^\S$ The University of Sydney \quad
  $^\ddag$ Carnegie Mellon University\\
  \vspace{0.2in}
  \href{https://causalcoat.github.io/legit}{\color{magenta}https://causalcoat.github.io/legit}
  \vspace{0.3in}
}

\maketitle
\customfootnotetext{$*$}{These authors contributed equally.}

\begin{abstract}
Designing proper experiments and selecting optimal intervention targets is a longstanding problem in scientific or causal discovery. 
Identifying the underlying causal structure from observational data alone is inherently difficult.
Obtaining interventional data, on the other hand, is crucial to causal discovery, yet it is usually expensive and time-consuming to gather sufficient interventional data to facilitate causal discovery.
Previous approaches commonly utilize uncertainty or gradient signals to determine the intervention targets. However, numerical-based approaches may yield suboptimal results due to the inaccurate estimation of the guiding signals at the beginning when with limited interventional data. 
In this work, we investigate a different approach, whether we can leverage Large Language Models (LLMs) to assist with the intervention targeting in causal discovery by making use of the rich world knowledge about the experimental design in LLMs.
Specifically, we present \oursfull (\ours) -- a robust framework that effectively incorporates LLMs to augment existing numerical approaches for the intervention targeting in causal discovery. 
Across $4$ realistic benchmark scales, \ours demonstrates significant improvements and robustness over existing methods and even surpasses humans, which demonstrates the usefulness of LLMs in assisting with experimental design for scientific discovery.

\end{abstract}

\section{Introduction}

Science originates along with discovering new causal knowledge with \textit{interventional experiments inspired by observations}~\citep{sci_revolution}. The art of finding causal relations from different interventions is then summarized and improved with statistical methods~\citep{book_why,auto_causal,Glymour2019ReviewOC}. Identifying and utilizing causal relations is fundamental to numerous applications, including biology~\citep{causal_learn_survey} and financial systems~\citep{Dong2023OnTT}. Despite the wide deployment of causal discovery methods, uncovering the underlying causal connections merely based on observational data alone is typically challenging due to limitations in identifiability. Mitigating this limitation usually requires additional interventional data obtained by perturbing part of the causal system to overcome the limited identifiability issue ~\citep{causation_2nd}.

However, collecting interventional data is expensive and time-consuming, as it usually involves a physical process of a real-world system~\citep{Cherry2012ReprogrammingCI}. Consequently, \textit{both the number of samples and the intervention targets are significantly limited in the experimental design in the real world}~\citep{Tong2001ActiveLF}. Previous approaches usually rely on uncertainty~\citep{Lindley1956OnAM} or information theoretic metric to maximize the utility of an experiment~\citep{tigas2022interventions}.
Recently, leveraging gradient signals for intervention targeting has gained significant success~\citep{olko2023git}, as it naturally fits into various gradient-based causal discovery methods.
Despite some success, both uncertainty-based and gradient-based approaches may still suffer from suboptimality, as the estimation of the signals is usually noisy. Especially when with limited interventional data, the inaccurate estimation of the scores can easily mislead the intervention targeting and the subsequent causal discovery.
The emergence of large language models (LLMs)~\citep{openai2023gpt4}, provides an opportunity to incorporate extensive world knowledge about experimental design into the intervention targeting process. 
It therefore raises an intriguing research question:

\begin{tcolorbox}[colframe=black, colback=white, coltitle=black,halign=center, valign=center]

\centering
\textit{Can we leverage LLMs for intervention targeting and do LLMs really help with it?}
\end{tcolorbox}

\begin{figure*}[tbp
]\vspace{-0.1in}
\begin{center} 
\includegraphics[width=1\textwidth]{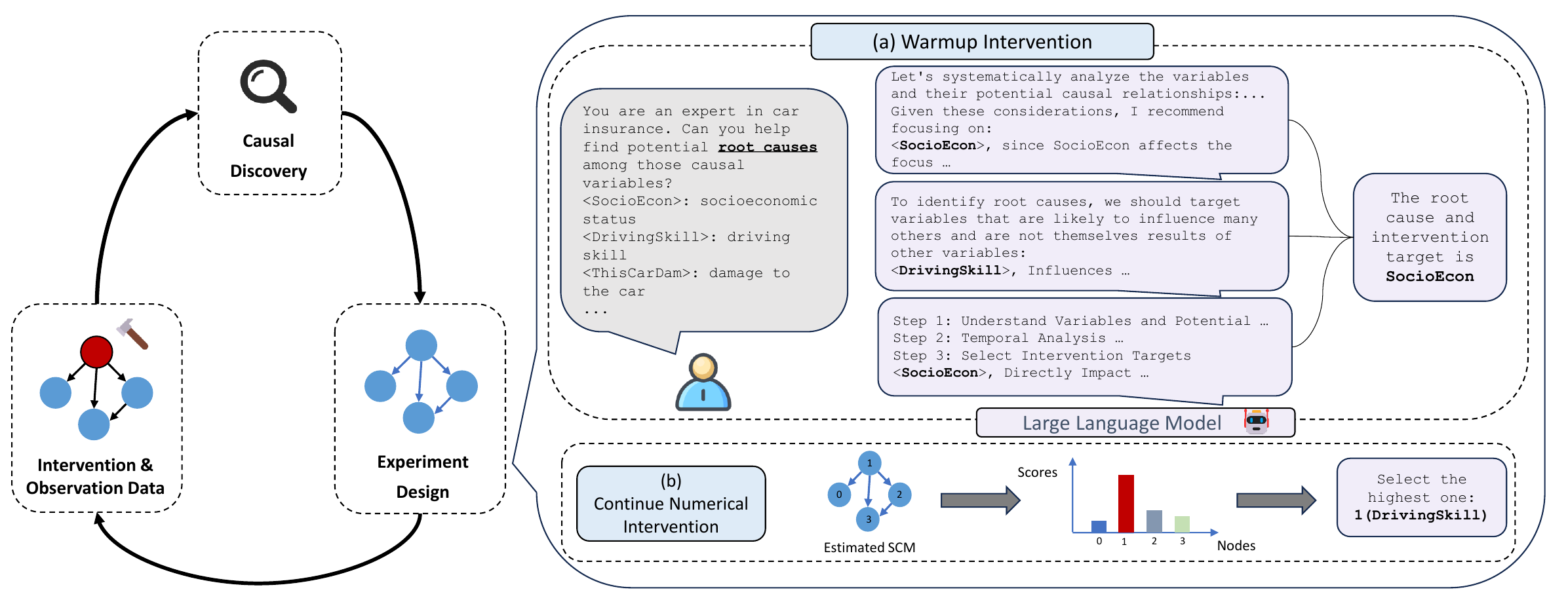}
\end{center}

\caption{Illustration of the \ours framework. The left side represents the loop of Online Causal Discovery, while the right side illustrates the experiment design process. In Step (a), Large Language Models (LLMs) warm up the causal discovery process by leveraging world knowledge and aligning it with the experiment's meta-information. This enables the identification of clear causal structures, which, in Step (b), guide previous methods to pinpoint informative intervention targets effectively.}
\label{fig:motivation}
\end{figure*}

Recent explorations into the use of LLMs for various causal learning and reasoning tasks suggest that these models may already encapsulate substantial domain knowledge~\citep{causal_llm_frontier,lampinen2023passive,cma}. 
LLMs have demonstrated the ability to process the meta-information encoded in natural language and leverage the meta-information to reason for the causality, which was considered restricted to humans~\citep{Gopnik2004ATO,Trott2022DoLL,Sahu2022UnpackingLL}. 
Furthermore, LLMs have exhibited remarkable potential in advancing complex scientific discovery~\citep{ai4science2023impact}.
Additionally, discussions about the limitations of LLMs in understanding causality were also raised in the community~\citep{causalparrots,corr2cau,understand_causality_llm}.
This underscores the need for a robust approach that optimally extracts the world knowledge embedded in LLMs about experimental design while mitigating the risks of being misled by their hallucinations regarding causality~\citep{llm_hallucination}.

To this end, we present a new framework called \oursfull (\ours), designed to maximize while robustly leveraging the knowledge in LLMs to assist with the intervention targeting. 
Shown as in Fig.~\ref{fig:motivation}, at the beginning of the causal discovery, the numerical-based methods have limited numerical knowledge about the underlying causal system to use due to the limited data. Consequently, the estimated signals tend to be noisy and misleading. In contrast, LLMs can leverage the meta-information about the causal system and relate the learned world knowledge to identify high-potential intervening targets.
After obtaining a relatively clearer causal graph, LLMs may not be able to provide sufficient guidance. Therefore, similar to humans, \ours leverages numerical methods to select the intervening targets.
Our contributions can be summarized as follows:
\begin{itemize}[leftmargin=*,itemsep=2pt, parsep=0pt, topsep=2pt]
    \item To the best of our knowledge, we are the first to investigate the use of LLMs in the experimental design to select intervention targets for causal discovery.
    \item We propose a novel framework called \ours that combines the advantages of both the previous numerical methods as well as the LLMs to facilitate the intervening targeting.
    \item We conduct extensive experiments with $4$ real-world benchmarks and verify that \ours can empower numerical-based methods and even human baseline.
    \item We highlight the promise of LLMs in causal and scientific discovery, that LLMs can effectively incorporate world knowledge, making them valuable cost-efficient complements to humans.
\end{itemize}

\section{Related Work}
\label{sec:related}

\paragraph{Intervention/Experiment Design} Scientific progress in causal discovery is often driven by interventional experiments inspired by observational insights~\citep{sci_revolution}. Traditional methods focused on designing effective experiments to establish causal links, while statistical approaches aimed to automate causal inference from observational data \citep{book_why,causation_2nd}. However, observational data alone is insufficient for identifying causal structures, and interventional data is costly to collect \citep{causation_2nd}. To address these challenges, several methods for optimal intervention design have been developed.

Active Intervention Targeting (AIT) selects intervention targets using an $F$-test inspired criterion, evaluating discrepancies in interventional sample distributions from a posterior distribution of graphs \citep{scherrer2021learning}. 
Causal Bayesian Experimental Design (CBED) uses Bayesian Optimal Experimental Design to select interventions that maximize mutual information (MI) between new data and existing graph beliefs, with MI estimated via a BALD-like method \citep{tigas2022interventions,houlsby2011}. 
Gradient-based Intervention Targeting (GIT)  \citep{olko2023git} leverages gradient information to determine interventions that maximize impact on causal parameter updates, which is particularly advantageous in low-data settings. 
In our work, we explore leveraging these advanced intervention strategies within the framework of LLMs to determine whether LLMs can effectively engage in experimental design for causal discovery, pushing the boundaries of what automated, data-driven causal inference can achieve.

\paragraph{Causal Discovery With LLMs} Recent advancements in large language models (LLMs) like ChatGPT have opened new opportunities in causal inference by incorporating domain knowledge, common sense, and contextual reasoning into the causal discovery process~\citep{causal_llm_frontier}. LLMs have demonstrated capabilities across Pearl's ladder of causation—association, intervention, and counterfactuals—bridging gaps that traditional models have with high-level causal reasoning. They have shown promising results in pairwise causal discovery tasks by utilizing semantic information not accessible through numerical data alone~\citep{jiralerspong2024efficient}.

Despite these advances, challenges remain. LLMs can sometimes behave like ``causal parrots'', repeating learned associations without demonstrating true causal reasoning~\citep{causalparrots}. Moreover, their performance varies significantly depending on task complexity, with limited success in advanced causal reasoning such as full graph discovery and counterfactual analysis~\citep{understand_causality_llm,corr2cau,imperfect_llm_expert}.
Another promising line of work integrates LLMs with traditional causal discovery methods to leverage their complementary strengths ~\citep{imperfect_llm_expert,cma,Liu2024DiscoveryOT}. This hybrid approach has shown improved performance in constructing causal graphs, benefiting from LLMs' understanding of language context and traditional methods' data-driven precision.

While prior studies emphasize the role of LLMs in causal analysis, the question of whether LLMs can meaningfully contribute to experimental design in causal discovery remains largely unaddressed. Experimental design encompasses proposing interventions, predicting outcomes, and assessing experimental strategies—tasks that extend beyond basic causal inference. This paper seeks to bridge this gap by investigating the potential of LLMs to support experimental design, exploring their unique value, and critically evaluating their strengths and limitations in guiding causal experiments.

\section{Preliminaries}
\label{sec:prelim}

We begin by briefly introducing the preliminaries and notations in the online causal discovery setting~\citep{olko2023git}.

\subsection{Causal Structure Discovery}
The causal relations between different variables can be formulated using the structural causal models (SCM)~\citep{book_why,causation_2nd,Glymour2019ReviewOC}. More specifically, in an SCM, we are given $n$ endogenous variables $X=(X_1,...,X_n)$, where the generation process of each variable can be expressed as $X_i=f_i(PA_i,U_i)$ where $PA_i$ is the set of variables that are the causal parents of $X_i$, and $U_i$ is the external independent noise when generating $X_i$.

The causal relations between $n$ variables can be further characterized via a direct acyclic graph (DAG), $G=(V,E)$, where $V=\{1,...,n\}$ is the set nodes corresponding to the set of random variables $\{X_1,...,X_n\}$. Each edge $(i,j)\in E$ in the edge set $E$ refers to the relation of direct cause $X_i\in PA_j$, i.e., $X_i$ is one of the causes of the variable $X_j$. The joint distribution of all the variables associated with the DAG can be expressed as $P(X_1,...,X_n)=\Pi_{i=1}^nP(X_i|PA_i)$.

Causal structure discovery aims to identify the underlying DAG $G$. However, when given only the joint observed distribution $P(X_1,...,X_n)$, it does not uniquely determine a DAG, as there might be different DAGs that can generate the same joint distribution. On the contrary, the observational data can merely determine a set of DAGs up to a Markov Equivalence Class (MEC)~\citep{causation_2nd}.

\begin{wrapfigure}{r}{0.68\textwidth}
    \begin{minipage}{0.68\textwidth}
    \vspace{-0.3in}
\begin{algorithm}[H]
\caption{\textsc{Online Causal Discovery}~\citep{olko2023git}}
\label{alg:online_causal_discovery}
\begin{algorithmic}[1]
    \INPUT{causal discovery algorithm $\mathcal{A}$ (e.g., ENCO,), intervention targeting {method $\mathcal{M}$},
    number of data acquisition rounds $T$, observational dataset $\mathcal{D}_{obs}$}
    \OUTPUT{final parameters of graph model: $\varphi_T$ and CausalDAG: $\mathbb{P}{(G)}$}
    \STATE $\mathcal{D}_{int} \gets \varnothing$
    \STATE Fit graph model $\varphi_0$ with algorithm  $\mathcal{A}$ on $\mathcal{D}_{obs}$ 
    \FOR{round $i = 1,2,\ldots,T$}
      \STATE $I \gets \text{ generate intervention targets using } {\mathcal{M}}$
      \STATE $\mathcal{D}_{int}^I \leftarrow \text{ query for data from interventions } I$
      \STATE $\mathcal{D}_{int} \leftarrow \mathcal{D}_{int} \cup \mathcal{D}_{int}^I$
      \STATE Fit $\varphi_i$ with algorithm  $\mathcal{A}$ \text{ on } $\mathcal{D}_{int} \text{ and } \mathcal{D}_{obs}$ 
    \ENDFOR
\end{algorithmic}
\end{algorithm}
\end{minipage}
\end{wrapfigure}

\subsection{Online Causal Discovery}
To identify the underlying ground truth DAG from the MEC, interventional data is widely incorporated into the causal discovery process~\citep{Tong2001ActiveLF,Hauser2011CharacterizationAG,Ke2019LearningNC}. Hence, online causal discovery is proposed to overcome the issue~\citep{Ke2019LearningNC}. 

As given in Algorithm~\ref{alg:online_causal_discovery}, an online causal discovery procedure is built upon a causal discovery algorithm $\gA$ that is able to leverage both the observational data and interventional data to recover the underlying causal structure.
More formally, the interventional data is usually obtained through single-node intervention on some causal variable $X_i$. The intervention will replace the generation process of $X_i$ with a new distribution, for which we denote as $\widehat{P}(X_i|PA_i)$~\citep{book_why}. Then, it yields an interventional distribution:
\begin{equation}
    P_i(X)=\widehat{P}(X_i|PA_i)\Pi_{j\neq i}P(X_j|PA_j),
\end{equation}

We use hard interventions in this study for simplicity and consistency, enabling causal structure identification and effectiveness assessment. Soft interventions, which adjust variable dependencies without removing them, are beyond this study's scope but may be explored in future research.

The online discovery will proceed by $T$ rounds. At the beginning of the first round, an initial graph model $\phi_0$ is fitted based on the observational data. Then, in the follow-up $T$ rounds, an intervention target $I$ will be selected using some intervention targeting method.
For each selected $I$, a batch of samples will be obtained and be integrated into all interventional data to execute the causal discovery algorithm $\gA$. After $T$ rounds, the fitted DAG will be the final output.

Previous approaches may use different intervention targeting methods. For example, \citep{scherrer2021learning} proposes Active Intervention Targeting (AIT) to select the desired intervention targets based on the $F$-test. \citep{tigas2022interventions} approximate the posterior distribution on all possible DAGs and leverage Bayesian Optimal Experimental Design to select the most informative intervention targets. 

\begin{figure}[t]
     \centering
    \includegraphics[width=0.9\textwidth]{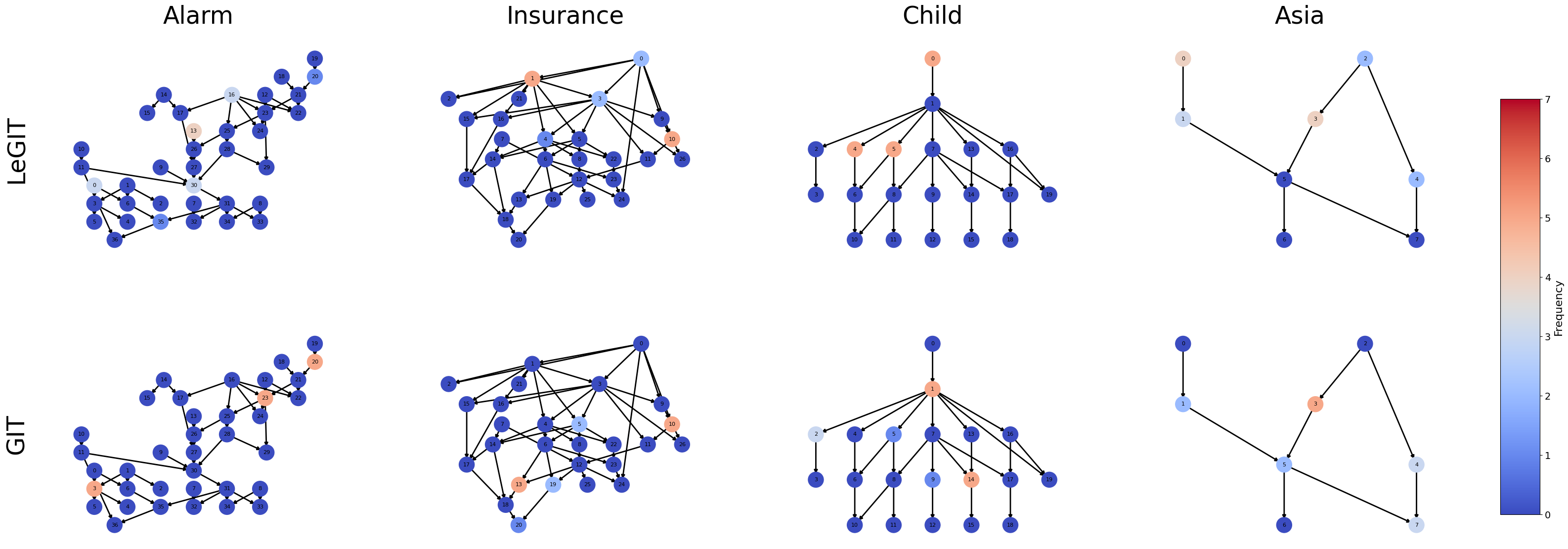}
    \caption{At the initial stage of the online causal discovery, the intervention targets from LLM-based selection and gradient-based selection.} 
    \label{fig:epoch1}
    \vspace{-0.2cm}
\end{figure}

Different from the Bayesian approaches, \citep{olko2023git} propose Gradient-based Intervention Targeting (GIT), which leverages the gradient signals from the gradient-based causal discovery methods to estimate the utility of each intervention target via hallucinated gradients~\citep{Ash2020Deep}. Due to the natural combination of the gradient-based causal discovery methods such as ENCO~\citep{lippe2022enco} and the GIT method, GIT achieves significant performance improvements over previous Bayesian-based approaches. Therefore, our follow-up discussion will center on the gradient-based approaches, i.e., the GIT method based on ENCO, as the state-of-the-art numerical method for online causal discovery.

\section{Methodology}
\label{sec:method}

\subsection{Challenges in Existing Intervention Targeting}

Despite the success of GIT method, similar to other estimation-based approaches, GIT is highly sensitive to the accuracy of the gradient estimation and estimated causal graphs which can be extremely noisy in the early rounds of an experiment. Therefore, we might mistakenly choose a variable that exerts minimal influence on the system, wasting valuable intervention budgets and misdirecting subsequent learning steps.

\begin{wrapfigure}{r}{0.5\textwidth}
    \centering
    \includegraphics[width=0.48\textwidth]{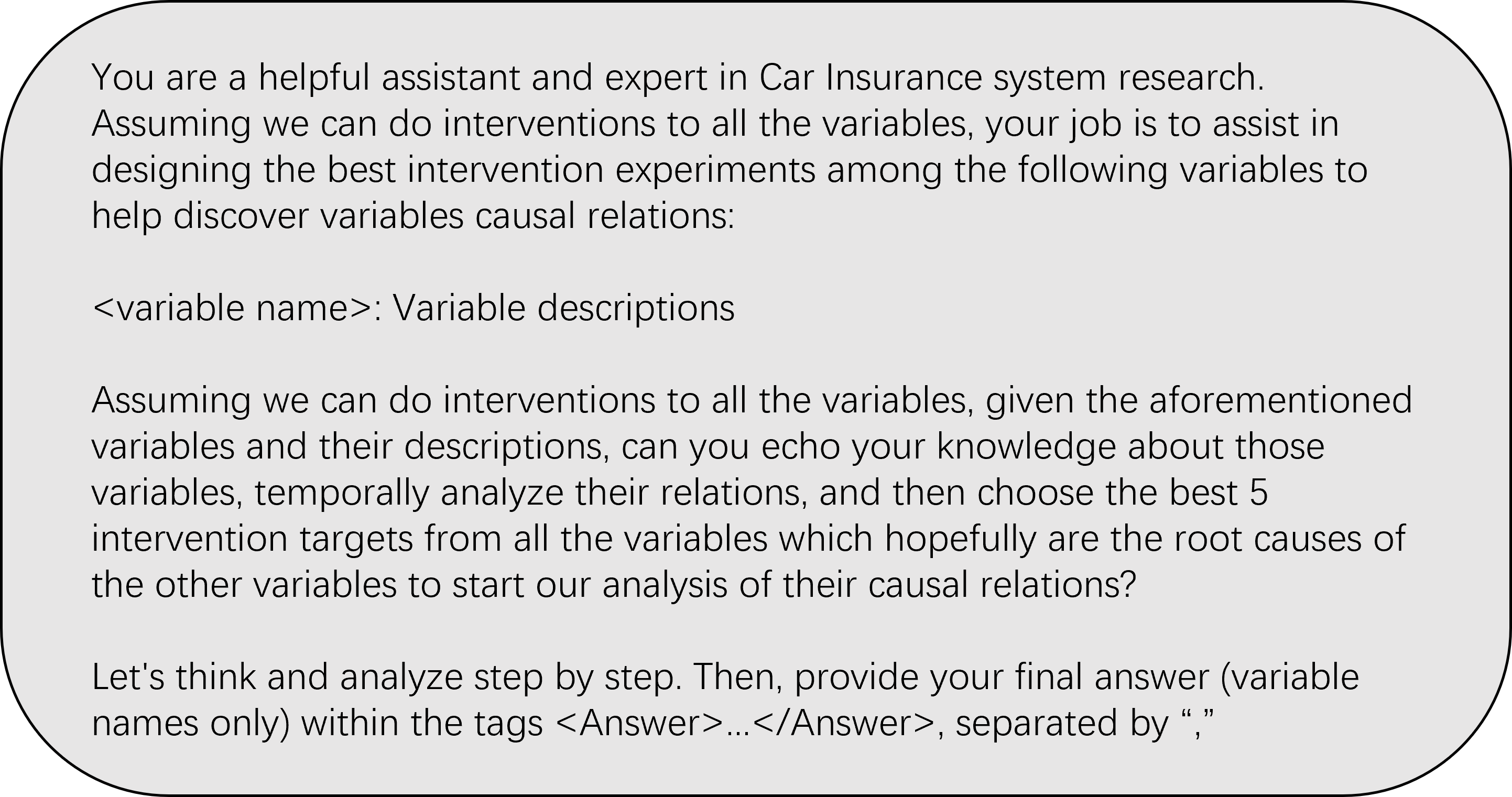}
    \caption{Prompt template at warmup stage.}
    \label{fig:prompt_template}
\end{wrapfigure}

To demonstrate the above issue and the challenges in the existing intervention targeting methods more concretely, we consider four realistic causal discovery benchmarks~\citep{JSSv035i03}, i.e., Alarm, Insurance, Child, and Asia and plot the distribution of the intervention target at the initial stage.

As given in Fig.~\ref{fig:epoch1}, it can be found that the success of GIT varies across different datasets. Intuitively, at the beginning of the intervention, intervening on variables that affect lots of other variables can bring more information about the system~\citep{Lindley1956OnAM}. 
In the Asia, Alarm dataset, the selected intervention targets are influential nodes. However, in insurance, the selected nodes only influence a few other nodes. Intervening on such targets with limited influence may lead to significant resource waste and further misdirect subsequent online causal discovery rounds.

In contrast, we construct prompts to inquire LLMs about the root causes in this system, given only the meta-information such as simple variable descriptions. The specific prompts are given in Fig.~\ref{fig:prompt_template}, and the suggested intervening targets are also highlighted in Fig.~\ref{fig:epoch1}. It can be found that given only the meta-information, LLMs are able to relate the rich world knowledge to locate the desired influential nodes.

\subsection{\oursfull}
Motivated by the aforementioned experiments, we present our framework \oursfull (\ours) to combine the strengths of both numerical-based methods and LLMs to facilitate the intervention targeting. The description of the algorithm of \ours is given in Algorithm~\ref{alg:legit}.
\ours consists of four stages.

\begin{figure}[htbp]
    \begin{minipage}{\textwidth}\vspace{-0.15in}
\begin{algorithm}[H]
\caption{\textsc{\ours: \oursfull} }
\label{alg:legit}
\begin{algorithmic}[1]
    \INPUT{Causal discovery algorithm for Intervention Data $\mathcal{A}$ (e.g., ENCO); Intervention Score targeting method $\mathcal{M_0}$ (e.g GIT); LLM for root cause proposal $\Psi$ , number of data acquisition rounds $T$; Observational dataset $\mathcal{D}_{obs}$; Graph Node List $V$; Warmup Epoch $T_{warmup}$; Bootstrapped Search Epoch $T_{bootstrapped}$}
    \OUTPUT{Final parameters of graph model: $\varphi_T$ and CausalDAG: $\mathbb{P}{(G)}$}
    \STATE \texttt{//Get Warmup List from LLM}\\
    $\mathcal{D}_{warmup} \gets \Psi(V,T_{warmup})$  
    \FOR{round $i = 1,2,\ldots,T$}
       \IF{i $<= T_{warmup}$}
        \STATE    $D_{int}^I \gets \mathcal{D}_{warmup}[i]$
       \ELSIF{i $= T_{warmup} +1$}
       \STATE \texttt{// Get the Isolated Nodes List}\\
            $V_{isolated} \gets$ isolated node from $\mathbb{P}{(G_i)}$ 
            \STATE \small \texttt{//Get Bootstrapped warmup Intervention Target from isolated Nodes}\\
            $\mathcal{D}_{bootstrapped} \gets \Psi(V_{isolated},T_{bootstrapped})$  
            \STATE $D_{int}^I \gets \mathcal{D}_{isolated}[i - T_{bootstrapped}]$
       \ELSIF{$T_{warmup} < i <= T_{warmup} + T_{bootstrapped}$}
            \STATE $D_{int}^I \gets \mathcal{D}_{bootstrapped}[i - T_{warmup}]$
        \ELSIF{$T_{warmup} + T_{missing} < i <= 2(T_{warmup} + T_{missing})$}
        \STATE \texttt{//Double Selection LLM'S List}\\
            $D_{int}^I \gets \mathcal({D}_{warmup}+D_{missing})[i - T_{warmup} - T_{missing}]$ 
        \ELSE
            \STATE $D_{int}^I \gets \text{ generate intervention targets using } {\mathcal{M_0}}$
       \ENDIF   
       \STATE$\mathcal{D}_{int} \leftarrow \mathcal{D}_{int} \cup \mathcal{D}_{int}^I$
      \STATE Fit $\varphi_i$ with algorithm  $\mathcal{A}$ \text{ on } $\mathcal{D}_{int} \text{ and } \mathcal{D}_{obs}$ 
    \ENDFOR
\end{algorithmic}
\end{algorithm}
\end{minipage}
\vspace{-0.1in}
\end{figure}

\paragraph{Warmup Stage} Since at the very beginning of the online causal discovery, numerical-based estimations are noisy and easily mislead the online causal discovery, we begin by prompting LLMs to relate the pre-trained knowledge, analyze the variable description, and suggest influential candidates. The prompt template is given in Fig.~\ref{fig:prompt_template}. The prompting will give the beginning list of intervention targets $\gD_{warmup}$. From $\gD_{warmup}$, we will select $T_{warmup}$ variables to obtain a basic map of the underlying causal system. For a robust performance, we perform self-consistency prompt skill~\citep{wang2022self} to get the final targets for a robust performance. 

\paragraph{Bootstrapped Stage} 

Although the first warmup stage yields a basic structure of the underlying causal system, due to the intrinsic limitations of LLMs such as limited context length~\citep{lost_in_middle} and hallucination~\citep{llm_hallucination}, LLMs may only focus on a subset of the variables and find the influential nodes therein. Nevertheless, when the number of causal variables is large, LLMs tend to give an incomplete set of influential nodes. 
Therefore, we further incorporate a second warmup stage, to bootstrap the use of LLM's world knowledge in early intervention targeting.

More concretely, we leverage the intermediate causal discovery results $\phi_{T_{warmup}}$ after the $T_{warmup}$ rounds and examine the left variables that have not been involved in $\phi_{T_{warmup}}$. Then, we further prompt LLMs to give more focus on the left set of variables and to find the influential variables that were missing in previous rounds. 

\paragraph{Double Selection Stage}

After getting the warmup and missing intervention targt, we perform a double selection to ensure the robustness of the discovered causal structure while minimizing unnecessary interventions ~\citep{lippe2022enco}.

\paragraph{Continual Intervention Stage} After the three warmup stages, we have already obtained relatively clearer yet complicated causal graphs. Even for humans, it is hard to determine the best experimental design. Therefore, we switch to using the numerical-based methods to continue to consume the remaining intervention budgets.

\vspace{-2mm}

\subsection{Theoretical Discussion}
\label{sec:llmocd_theory}
After setting up the \ours algorithm, we discuss the convergence of \ours. 
Since \ours ends up with a numerical-based method for concluding online causal discovery, it follows intuitively that, like other numerical-based methods (e.g., GIT~\citep{olko2023git}), and an effective causal discovery algorithm, such as ENCO~\citep{lippe2022enco}, \ours can converge, further details available in Appendix. Nevertheless, due to its enhanced strategy in \ours, we empirically observe that \ours can converge to a better solution when compared to the same numerical-based method without LLMs involved.

\subsection{Practical Discussion}
\label{sec:llmocd_impl}
Consistent with prior work, we mainly adopt GIT as the numerical-based method $\gM$, and ENCO as the gradient-based causal discovery method. However, as also suggested in GIT~\citep{olko2023git}, ENCO can also be switched to other gradient-based methods. Additionally, \ours is also compatible with other numerical-based approaches.

\section{Experiments}

In this section, we conduct extensive experiments to evaluate \ours on real-world datasets and compare \ours against various baselines in intervention selection and humans. We provide a brief overview of the experimental setups here, with further details available in Appendix.

\subsection{Experimental Setup}
\label{section:Exp}

\paragraph{Datasets} Specifically, we use four real-world benchmark datasets along with their corresponding ground truth causal graphs from the BN repository~\citep{JSSv035i03}: \textit{Asia}, \textit{Child}, \textit{Insurance}, and \textit{Alarm}. It provides causal graphs derived from real-world applications that are widely recognized as benchmark datasets. These datasets encompass a diverse set of professional scenarios, ranging from car insurance to medical systems, which are crucial for enhancing the knowledge captured by large language models (LLMs).
\begin{enumerate}
\vspace{-0.01in}
\item    \textit{Asia}~\citep{asia} dataset consists of 8 variables related to a lung cancer diagnosis system, with 8 edges.
\vspace{-0.01in}
\item  \textit{Child}~\citep{CHILD} dataset contains 20 nodes and 25 edges, modeling congenital heart disease in newborns.
\vspace{-0.01in}
\item  \textit{Insurance}~\citep{insurance} dataset includes 27 nodes and 52 edges, representing a car insurance system.
\vspace{-0.01in}
\item  \textit{Alarm}~\citep{Beinlich1989TheAM} dataset comprises 37 nodes and 46 edges, simulating an alarm message system for patient monitoring.
\end{enumerate}
\paragraph{Baselines}
We compare \textbf{\ours} against different online causal discovery algorithms \textbf{GIT}~\citep{olko2023git}, \textbf{AIT}~\citep{scherrer2021learning}, \textbf{CBED}~\citep{tigas2022interventions} as selection strategies for online active learning interventions, as well as three random baselines and human baseline:
\begin{enumerate}
\vspace{-0.05in}
\item    \textbf{Random Choice}:
A target node is selected uniformly at random from the set of all nodes at each step;
\vspace{-0.05in}
\item  \textbf{Round Robin}: 
A target node is chosen randomly from the unvisited nodes at each step. Once all nodes are selected, the visitation counts are reset.
\vspace{-0.05in}
\item  \textbf{Degree Prob Sample}:
A target node is randomly chosen from all nodes, with selection probability normalized by each node's out-degree.
\vspace{-0.05in}
\item \textbf{Human}: 
We engaged five master's/Ph.D.-level individuals, presenting them with the same information and process as provided to the LLMs.
\end{enumerate}
Among the baselines, Degree Prob Sample can be considered as an oracle to LLM that adopts the out-degree of each node in the ground truth DAG. And compare the human analysts to find the unique contributions of LLMs.

\begin{table}[!t]
\centering
\caption{Average SHD, SID, and BSF with standard deviation (over 5 seeds) for real-world data ($T = 33$ rounds, $|D_{int}^I| = 32, N = 1056$).}
\resizebox{\linewidth}{!}{
\begin{tabular}{l ccc ccc ccc ccc}
\toprule
\multirow{2}{*}{Methods}
& \multicolumn{3}{c}{\textbf{Alarm}} 
& \multicolumn{3}{c}{\textbf{Insurance}} 
& \multicolumn{3}{c}{\textbf{Child}} 
& \multicolumn{3}{c}{\textbf{Asia}} \\
\cmidrule(lr){2-4}\cmidrule(lr){5-7}\cmidrule(lr){8-10}\cmidrule(lr){11-13}
& \textbf{SHD$\downarrow$} & \textbf{SID$\downarrow$} & \textbf{BSF$\uparrow$}  
& \textbf{SHD$\downarrow$} & \textbf{SID$\downarrow$} & \textbf{BSF$\uparrow$}  
& \textbf{SHD$\downarrow$} & \textbf{SID$\downarrow$} & \textbf{BSF$\uparrow$}  
& \textbf{SHD$\downarrow$} & \textbf{SID$\downarrow$} & \textbf{BSF$\uparrow$} \\
\midrule
CBED          
& 28.20 {\scriptsize ± 4.31}   & 213.80 {\scriptsize ± 42.44}   & 0.8053 {\scriptsize ± 0.06}
& 21.60 {\scriptsize ± 4.63}   & 260.00 {\scriptsize ± 31.83}   & 0.7529 {\scriptsize ± 0.04}
& 5.40 {\scriptsize ± 2.06}    & 44.40 {\scriptsize ± 18.51}    & 0.9150 {\scriptsize ± 0.05}
& 2.20 {\scriptsize ± 1.47}    & 4.60 {\scriptsize ± 3.38}      & 0.7833 {\scriptsize ± 0.08} \\[2pt]

AIT           
& 32.80 {\scriptsize ± 8.42}   & 204.60 {\scriptsize ± 52.09}   & 0.7214 {\scriptsize ± 0.05}
& 24.20 {\scriptsize ± 7.47}   & 312.40 {\scriptsize ± 87.50}   & 0.6711 {\scriptsize ± 0.11}
& 9.00 {\scriptsize ± 3.29}    & 52.20 {\scriptsize ± 21.03}    & 0.8752 {\scriptsize ± 0.04}
& 1.80 {\scriptsize ± 0.75}    & 6.60 {\scriptsize ± 4.84}      & 0.7833 {\scriptsize ± 0.08} \\[2pt]

Random Choice 
& 38.80 {\scriptsize ± 3.54}   & 204.40 {\scriptsize ± 58.15}   & 0.7430 {\scriptsize ± 0.08}
& 26.00 {\scriptsize ± 3.63}   & 323.80 {\scriptsize ± 14.96}   & 0.7137 {\scriptsize ± 0.02}
& 5.40 {\scriptsize ± 1.20}    & 51.00 {\scriptsize ± 17.11}    & 0.9396 {\scriptsize ± 0.04}
& 1.20 {\scriptsize ± 0.40}    & 2.20 {\scriptsize ± 2.40}      & 0.8708 {\scriptsize ± 0.01} \\[2pt]

Round Robin   
& 25.00 {\scriptsize ± 1.26}   & \textbf{118.60 {\scriptsize ± 21.78}}  & 0.9301 {\scriptsize ± 0.02}
& 17.40 {\scriptsize ± 4.54}   & \underline{232.20 {\scriptsize ± 27.23}}           & \underline{0.8042 {\scriptsize ± 0.02}}
& 3.40 {\scriptsize ± 2.50}    & 23.00 {\scriptsize ± 14.39}            & 0.9824 {\scriptsize ± 0.02}
& 1.40 {\scriptsize ± 0.49}    & 2.20 {\scriptsize ± 1.47}              & 0.8250 {\scriptsize ± 0.06} \\[2pt]

Degree Prob   
& 29.40 {\scriptsize ± 4.67}   & 144.60 {\scriptsize ± 49.77}   & 0.7798 {\scriptsize ± 0.06}
& 25.80 {\scriptsize ± 2.93}   & 305.20 {\scriptsize ± 17.45}   & 0.7054 {\scriptsize ± 0.03}
& 6.20 {\scriptsize ± 2.48}    & 36.20 {\scriptsize ± 16.35}    & 0.8842 {\scriptsize ± 0.06}
& \underline{1.00 {\scriptsize ± 0.00}}    & \underline{1.00 {\scriptsize ± 0.00}}      & \underline{0.8750 {\scriptsize ± 0.00}} \\[2pt]

GIT           
& \underline{19.60 {\scriptsize ± 3.77}}  & 131.40 {\scriptsize ± 47.66}   & \underline{0.9295 {\scriptsize ± 0.02}}
& \underline{16.40 {\scriptsize ± 3.14}}  & 243.80 {\scriptsize ± 28.72}   & 0.7960 {\scriptsize ± 0.02}
& \underline{2.80 {\scriptsize ± 0.75}}   & \underline{20.40 {\scriptsize ± 12.50}}    & \underline{0.9841 {\scriptsize ± 0.01}}
& \underline{1.00 {\scriptsize ± 0.00}}   & \underline{1.00 {\scriptsize ± 0.00}}      & \underline{0.8750 {\scriptsize ± 0.00}} \\[2pt] 
\midrule
\textbf{Human}
& 22.60 {\scriptsize ± 5.43}   & 133.20 {\scriptsize ± 27.01}   & 0.8976 {\scriptsize ± 0.02}
& \underline{14.20 {\scriptsize ± 3.43}}  & 232.20 {\scriptsize ± 40.74}   & 0.8065 {\scriptsize ± 0.03}
& \textbf{2.00 {\scriptsize ± 0.63}}      & \textbf{18.80 {\scriptsize ± 8.42}}   & \textbf{0.9944 {\scriptsize ± 0.00}}
& 1.40 {\scriptsize ± 0.49}    & 3.60 {\scriptsize ± 3.20}      & 0.8667 {\scriptsize ± 0.01} \\[2pt]

\rowcolor{SkyBlue!20}
\textbf{\ours} 
& \textbf{17.40 {\scriptsize ± 3.61}}   & \textbf{121.00 {\scriptsize ± 38.27}}   & \textbf{0.9362 {\scriptsize ± 0.02}}
& \textbf{12.60 {\scriptsize ± 0.80}}   & \textbf{200.60 {\scriptsize ± 35.32}}   & \textbf{0.8205 {\scriptsize ± 0.01}}
& \textbf{2.20 {\scriptsize ± 0.98}}    & \textbf{20.60 {\scriptsize ± 5.61}}     & 0.9858 {\scriptsize ± 0.02}
& \textbf{0.80 {\scriptsize ± 0.75}}    & \textbf{0.80 {\scriptsize ± 0.40}}      & \textbf{0.9000 {\scriptsize ± 0.05}} \\
\bottomrule
\end{tabular}
}
\label{table:shd_and_sid}
\end{table}

\paragraph{Implementation}
We employ the GPT-4O API \footnote{We used gpt-4o-2024-08-06 version.}~\citep{openai2024gpt4o} for all LLM-based experiments. In all experiments presented in this section, we utilize \textbf{ENCO}~\citep{lippe2022enco} as the backbone causal discovery algorithm, with detailed settings provided in the Appendix. The observational dataset consists of $|\mathcal{D}{obs}| = 5000$ samples, and we conduct $T = 33$ rounds of intervention sampling, with each round acquiring an interventional batch of $|\mathcal{D}_{int}^I| = 32$ samples, leading to a total of $N = 1056$ interventional samples. For both \textbf{GIT} and \textbf{AIT}, we use $|\mathcal G| = 50$ graphs, each with $|\mathcal{D}_{G, i}| = 128$ data samples for the Monte Carlo approximation of the score.

Considering the size of the real-world graph, we set $T_{warmup} = 3$ and $T_{bootstrapped} = 2$ in \ours, with the exception of the Asia dataset, which has a smaller size. For the Asia dataset, we set $T_{warmup} = 3$ and $T_{bootstrapped} = 1$.

\paragraph{Metrics} 

We evaluate the performance of the online causal discovery algorithms using three metrics: the Structural Hamming Distance (SHD)\citep{shd}, the Structural Intervention Distance (SID)\citep{peters2015structural}, and BSF~\citep{constantinou2019evaluating}. SHD (lower is better) quantifies the number of edge insertions, deletions, or reversals needed to transform one graph into another. SID (lower is better) assesses causal inference by evaluating the correctness of the intervention distribution. BSF (higher is better) mitigates bias by balancing the evaluation of edges and independencies within Bayesian Network structures. A detailed description of these metrics can be found in the Appendix.

\subsection{Empirical Results}

\begin{figure}[t]
     \centering
     \includegraphics[width=1\textwidth]{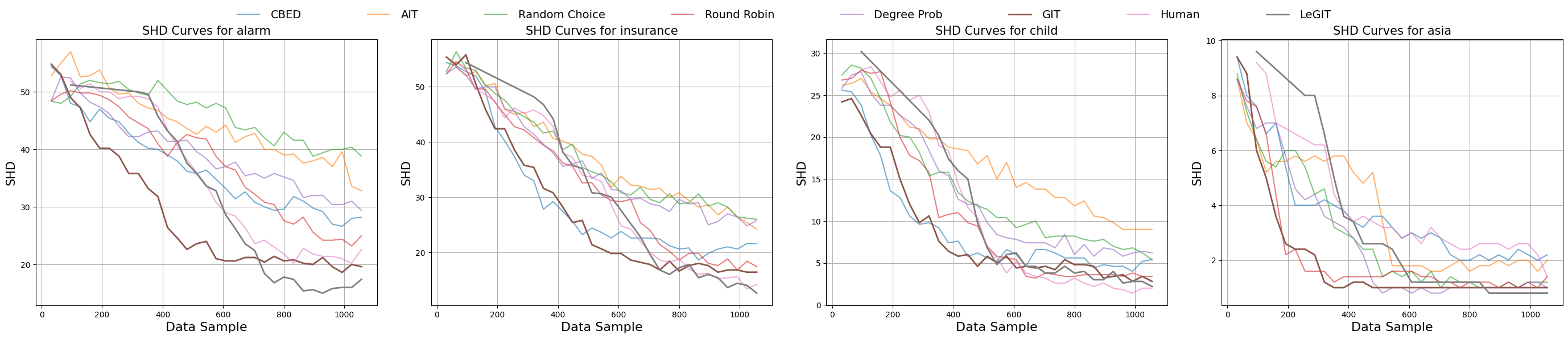}
    \caption{SHD metric for different methods (over 5 seeds) towards different intervention samples. ($T = 33$ rounds, $|D_{int}^I| = 32, N = 1056$)} 
    \label{fig:datasamples}
\end{figure}

The results of the experiments are presented in Table~\ref{table:shd_and_sid}. 
Our method consistently outperforms the baseline approaches across four distinct domains, as evidenced by SHD calculated from five seeds. 
Figure~\ref{fig:datasamples} illustrates the mean SHD of these methods in relation to the number of intervention samples.

As shown in Table~\ref{table:shd_and_sid}, it can be found that \ours achieves state-of-the-art causal discovery performances, with consistent improvements against the adopted gradient-based methods and human baseline.
The superior SHD scores demonstrate that \ours is highly effective in accurately reconstructing the underlying graph structures, minimizing the number of erroneous edge modifications required. 
In Fig.~\ref{fig:datasamples}, we further plot the performances of different methods along with the increase of the data samples obtained from the intervention. 
It can be found that, although at the beginning of the online causal discovery, \ours may not demonstrate outstanding SHD results. Along with more data samples combining, \ours converge to a better solution faster than any other methods. In contrast, despite a faster decrease speed of GIT, GIT finally converges to a suboptimal solution due to unsuitable initialization, which verifies our discussion.
real-world
Besides, SID results highlight \ours's robustness in preserving causal relationships and ensuring accurate causal inferences, which is essential in real-world applications. For BSF metrics, higher values are indicative the learned graph is more accurate and closely matches the true graph in terms of structure and dependencies. 

The results clearly indicate that \ours outperforms existing baseline methods across all three evaluation metrics. The consistently low SHD and SID scores, coupled with high BSF values, underscore the efficacy of \ours in accurately learning network structures and providing tangible benefits. Compared to heuristic-based methods like Random Choice and Round Robin, \ours offers a more strategic and data-driven approach, leading to better performance metrics.

Moreover, while Human interventions remain strong competitors, \ours bridges the gap between automated methods and expert-driven processes. This positions \ours as an effective tool for structure learning, capable of delivering expert-level performance without the need for manual interventions.

For more details on the node variable descriptions and the out-degree distribution, please refer to the Appendix.

\subsection{Low Data Experiment Analysis}

\begin{table*}[ht!]
\centering
\caption{Average SHD, SID, and BSF with standard deviation (over 5 seeds) for real-world data with a low data budget ($T = 33$ rounds, $|D_{int}^I| = 16,\ N = 528$).}
\resizebox{\linewidth}{!}{
\begin{tabular}{l ccc ccc ccc ccc}
\toprule
\multirow{2}{*}{Methods}
& \multicolumn{3}{c}{\textbf{Alarm}} 
& \multicolumn{3}{c}{\textbf{Insurance}} 
& \multicolumn{3}{c}{\textbf{Child}} 
& \multicolumn{3}{c}{\textbf{Asia}} \\
\cmidrule(lr){2-4}\cmidrule(lr){5-7}\cmidrule(lr){8-10}\cmidrule(lr){11-13}
& \textbf{SHD$\downarrow$} & \textbf{SID$\downarrow$} & \textbf{BSF$\uparrow$}  
& \textbf{SHD$\downarrow$} & \textbf{SID$\downarrow$} & \textbf{BSF$\uparrow$}  
& \textbf{SHD$\downarrow$} & \textbf{SID$\downarrow$} & \textbf{BSF$\uparrow$}  
& \textbf{SHD$\downarrow$} & \textbf{SID$\downarrow$} & \textbf{BSF$\uparrow$} \\
\midrule

CBED
& 32.40 {\scriptsize ± 4.36} & 214.20 {\scriptsize ± 69.73} & 0.8229 {\scriptsize ± 0.05}
& 26.40 {\scriptsize ± 3.56} & 327.00 {\scriptsize ± 38.46} & 0.7042 {\scriptsize ± 0.02}
& 9.20  {\scriptsize ± 3.25} & 46.60  {\scriptsize ± 18.49} & 0.8735 {\scriptsize ± 0.06}
& 2.60  {\scriptsize ± 2.73} & 4.40   {\scriptsize ± 2.94}  & 0.7500 {\scriptsize ± 0.18} \\[2pt]

AIT
& 41.20 {\scriptsize ± 5.49} & 270.00 {\scriptsize ± 29.61} & 0.6728 {\scriptsize ± 0.07}
& 37.00 {\scriptsize ± 12.26} & 421.40 {\scriptsize ± 82.68} & 0.5602 {\scriptsize ± 0.12}
& 10.00 {\scriptsize ± 3.29} & 73.40  {\scriptsize ± 45.64} & 0.8175 {\scriptsize ± 0.11}
& 1.40  {\scriptsize ± 0.49} & 3.60   {\scriptsize ± 2.33}  & 0.8417 {\scriptsize ± 0.05} \\[2pt]

Random Choice
& 40.80 {\scriptsize ± 2.71} & 236.40 {\scriptsize ± 12.31} & 0.7203 {\scriptsize ± 0.07}
& 25.60 {\scriptsize ± 2.24} & 311.00 {\scriptsize ± 22.17} & 0.7211 {\scriptsize ± 0.03}
& 8.20  {\scriptsize ± 2.32} & 51.60  {\scriptsize ± 33.15} & 0.8300 {\scriptsize ± 0.07}
& \textbf{1.20  {\scriptsize ± 0.40}} & \textbf{2.40   {\scriptsize ± 2.80}}  & \textbf{0.8458 {\scriptsize ± 0.06}} \\[2pt]

Round Robin
& 33.60 {\scriptsize ± 7.34} & \underline{169.00 {\scriptsize ± 35.69}} & 0.8975 {\scriptsize ± 0.06}
& 22.60 {\scriptsize ± 3.72} & \underline{269.20 {\scriptsize ± 44.37}} & \underline{0.7749 {\scriptsize ± 0.04}}
& \underline{4.60  {\scriptsize ± 2.42}} & \underline{32.40  {\scriptsize ± 24.25}} & \underline{0.9327 {\scriptsize ± 0.06}}
& \textbf{1.20  {\scriptsize ± 0.40}} & \textbf{2.40   {\scriptsize ± 2.80}}  & \textbf{0.8458 {\scriptsize ± 0.06}} \\[2pt]

Degree Prob
& 42.60 {\scriptsize ± 6.34} & 244.20 {\scriptsize ± 35.06} & 0.6762 {\scriptsize ± 0.08}
& 31.80 {\scriptsize ± 4.40} & 351.00 {\scriptsize ± 27.64} & 0.6737 {\scriptsize ± 0.05}
& 9.00  {\scriptsize ± 2.90} & 60.80  {\scriptsize ± 23.01} & 0.8512 {\scriptsize ± 0.07}
& 1.40  {\scriptsize ± 0.49} & 3.60   {\scriptsize ± 3.20}  & 0.8667 {\scriptsize ± 0.01} \\[2pt]

GIT
& \underline{27.20 {\scriptsize ± 4.71}} & \underline{177.80 {\scriptsize ± 61.65}} & \underline{0.9025 {\scriptsize ± 0.04}}
& \underline{22.40 {\scriptsize ± 3.72}} & \underline{296.00 {\scriptsize ± 44.23}} & \underline{0.7463 {\scriptsize ± 0.04}}
& 6.00  {\scriptsize ± 1.55} & \underline{33.80  {\scriptsize ± 15.75}} & \underline{0.9134 {\scriptsize ± 0.05}}
& 1.80  {\scriptsize ± 0.75} & 6.00   {\scriptsize ± 4.56}  & 0.8333 {\scriptsize ± 0.06} \\[2pt]
\midrule
\rowcolor{SkyBlue!20}
\textbf{\ours} 
& \textbf{21.00 {\scriptsize ± 2.37}} & \textbf{159.40 {\scriptsize ± 26.81}} & \textbf{0.9158 {\scriptsize ± 0.01}}
& \textbf{18.20 {\scriptsize ± 1.17}} & \textbf{259.00 {\scriptsize ± 66.69}} & \textbf{0.7894 {\scriptsize ± 0.02}}
& \textbf{4.40  {\scriptsize ± 2.15}} & \textbf{28.20  {\scriptsize ± 15.03} }& \textbf{0.9499 {\scriptsize ± 0.04}}
& 1.40  {\scriptsize ± 0.49} & 3.80   {\scriptsize ± 3.43}  & 0.8417 {\scriptsize ± 0.06} \\[2pt]
\bottomrule
\end{tabular}
}
\label{table:lowdata_full}
\vspace{-3mm}
\end{table*}

Furthermore, we conduct additional experiments in an extremely low-data setting, where only $16$ interventional data samples are sampled from each round, other settings are the same as above.
This low-data setting is more practically relevant. Additionally, due to the insufficient intervention data, the performance of causal discovery algorithms in estimating effects is diminished~\citep{lippe2022enco}, which further tests the effectiveness of the intervention strategy.

The results presented in Table~\ref{table:lowdata_full}, show that \ours achieves larger improvements under these conditions in 3 complex datasets. These findings highlight the effectiveness of \ours in real-world experimental design scenarios, where both the number of interventions and the sample size are limited.

The result of the low-data experiment further verifies our discussion that numerical methods suffer from noise or insufficient data, leading to a suboptimal solution. The numerical-based method does not even outperform round-robin on 3 smaller datasets, underscoring its limitations in such scenarios.
In contrast, the use of LLMs enables scalable and effective guidance that complements numerical methods, reducing the risk of suboptimal convergence, and having more stable performance in real-world applications.

\begin{figure}[h]
     \centering
     \includegraphics[width=\textwidth]{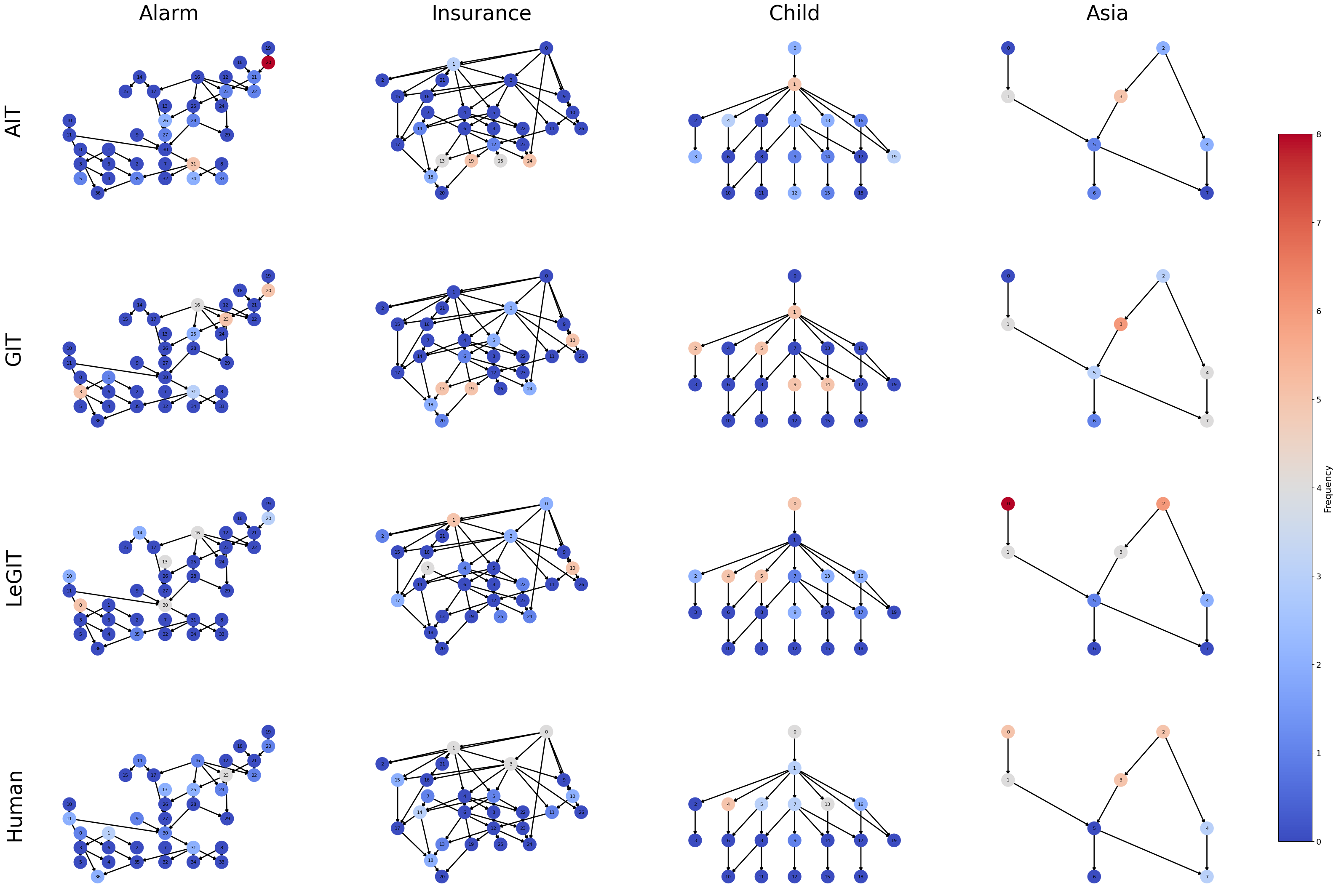}
    \caption{The selected Node Frequency obtained by different strategies on Epoch 0-4 from 5 different seeds under Table\ref{table:shd_and_sid} setting.} 
    \label{fig:epoch5}
    \vspace{-0.1in}
\end{figure}

\subsection{Detailed comparisons and analyses}

Figs.~\ref{fig:epoch5} depict the distribution of selected nodes between epochs 0-5. Notably, the numerical methods (GIT, AIT) tend to get trapped in the initialization phase of the Insurance, alarm dataset, consistently selecting less central nodes in the graph, often peripheral or leaf nodes. In contrast, our model (\ours) that SocioEcon (socioeconomic status, No.1 nodes in graph), plays a crucial role in the insurance system, potentially influencing car choice, driving behavior, and the ability to afford certain safety features. 

Compared to the Human baseline, \ours demonstrates superior performance on two complex datasets: Alarm and Insurance. As the number of variables increases, determining the optimal interventions to reveal the structure of the causal graph becomes combinatorially explosive. For humans, this process can be extremely tedious or error-prone, as they may subjectively favor certain nodes, failing to synthesize different viewpoints due to simpler mental models. In contrast, refer to Figs.\ref{fig:motivation}, LLMs follow the instructions provided in Fig.3 step by step and align them with their background knowledge. With the self-consistency prompt technique, LLMs generate more robust results, providing a highly cost-effective alternative to hiring multiple human experts for advice. 

LLMs' primary value lies in scalability and availability, providing immediate, cost-effective guidance in real-time, especially for online causal discovery where rapid interventions are required. They excel in large-scale systems with many variables, where it’s infeasible for experts to assess all nodes. LLMs complement human oversight by filling gaps in availability, consistency, and knowledge while helping avoid expert biases. Additionally, LLMs quickly process metadata, saving experts time and providing a solid starting point, as seen in other AI-assisted tasks.

\vspace{-2mm}
\section{Conclusions}
In this work, we investigated how to incorporate LLMs into the intervention targeting in experimental design for causal discovery.
We introduced a novel framework called \ours, which combines the best of previous numerical-based approaches and the rich knowledge in LLMs. 
Specifically, \ours leverages LLMs to warm up the online causal discovery procedure by identifying the influential root cause variables to begin the intervention. 
After setting up a relatively clear picture of the underlying causal graph, \ours then integrates the numerical-based methods to continue to select the intervention targets. Empirically, we verified the effectiveness of \ours leveraging LLMs to warm up the online causal discovery can achieve the state-of-the-art performance across multiple realistic causal discovery benchmarks. 
Furthermore, we compared its performance against a human baseline, highlighting its unique value.
LLMs offer a scalable and cost-effective approach to enhance experimental design, paving the way for new research directions of causal analysis and scientific discovery fields.

\bibliographystyle{references}
\bibliography{references/causality,references/llm,references/causalllm,references/onlineCD,references/exp_design}
\appendix
\clearpage

\appendix


\section{More Details of Datasets}
\label{appdx:datsets}
In this part, we will further introduce the 4 different domain Causal graph discovery dataset from BNleaner Repository~\citep{JSSv035i03}. For the description of each variable, we refer to \citep{llm_build_graph} and make some changes on it. We show the ground truth and the out-degree nodes distributions as follows.

\textbf{Asia} show as Fig.\ref{fig:asia} aims to model a hypothetical medical scenario in which a person visits a clinic with shortness of breath. The network helps in diagnosing the likely causes (e.g., tuberculosis, lung cancer, bronchitis) by probabilistically combining the available evidence (e.g., history of travel, smoking status, X-ray results)

\textbf{Child} show as Fig.\ref{fig:child} is used to model the diagnosis of pediatric health issues, particularly those that can occur in newborns or young children. It’s often employed in studies related to decision support systems, where probabilistic graphical models assist in medical diagnosis. The network is significantly larger than the Asia dataset, with 20 nodes (variables) and 25 edges. 

\textbf{Insurance} shown as Fig.~\ref{fig:insurance} intended to simulate a situation in which an insurance company needs to assess various risks and make decisions regarding policies, claims, and customer behavior. It represents the interdependencies between multiple insurance factors. It has 27 nodes and 52 edges

\textbf{Alarm} shown as Fig.~\ref{fig:alarm} is known as the ALARM (A Logical Alarm Reduction Mechanism) network, and it was originally developed to model a patient monitoring system for anesthesia purposes. It helps in predicting physiological conditions of patients, detecting potential complications, and generating alerts when necessary, consists of 37 nodes and 46 edges.


\begin{figure}[!h]
    \centering
    \includegraphics[width=0.9\linewidth]{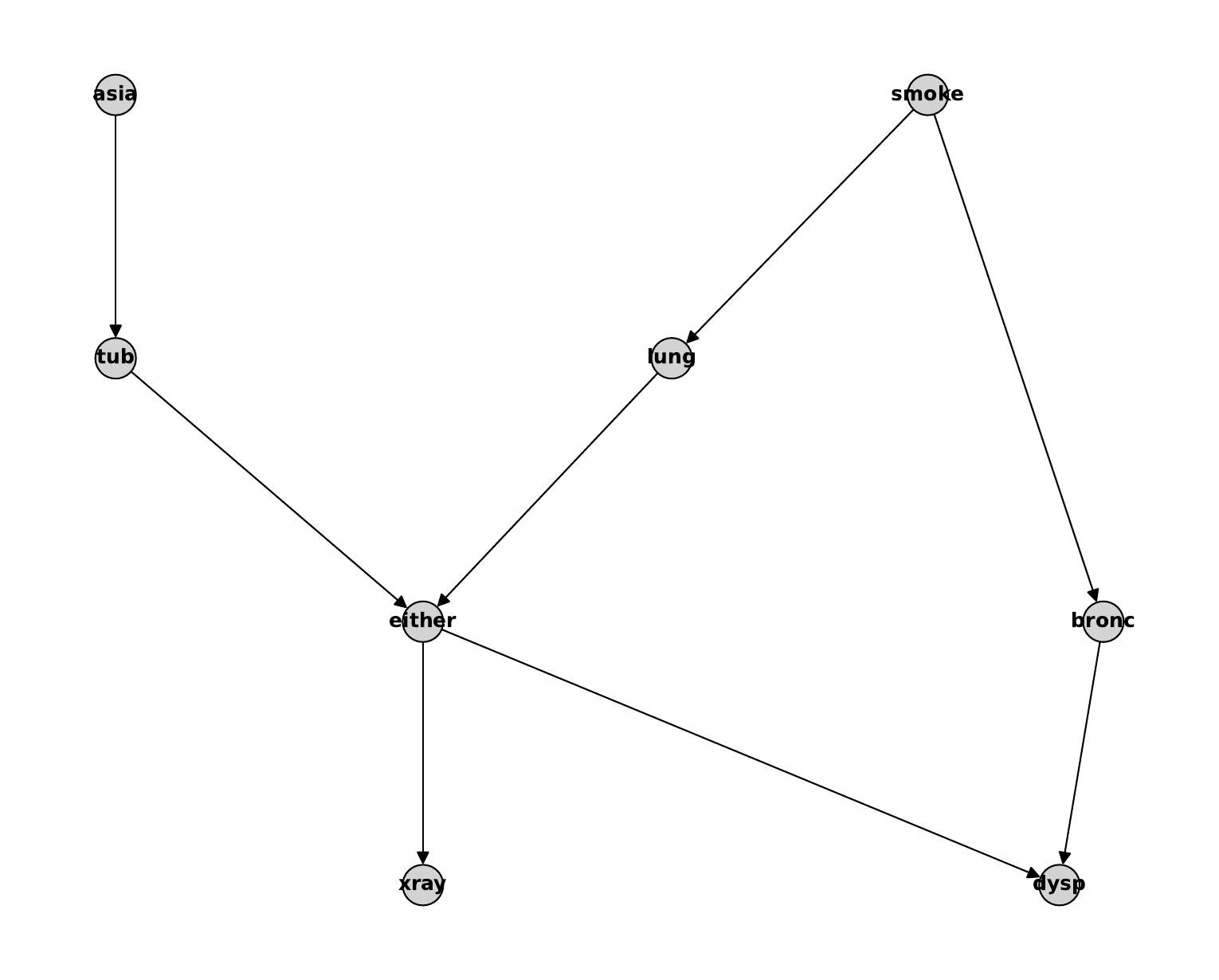}
    \caption{Ground truth Causal Graph for asia data.}
    \label{fig:asia}
\end{figure}

\begin{figure}[!h]
    \centering
    \includegraphics[width=0.8\linewidth]{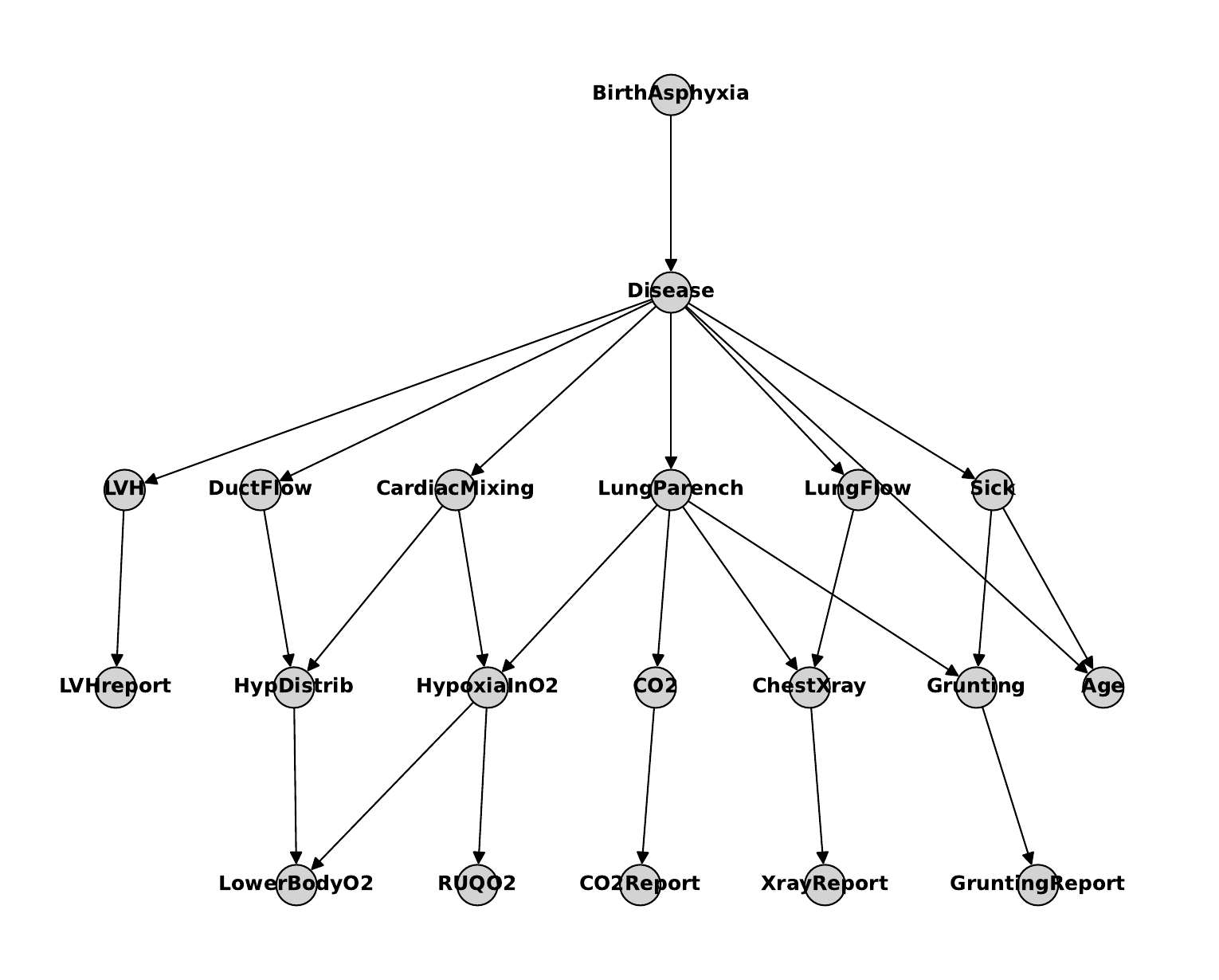}
    \caption{Ground truth Causal Graph for child data.}
    \label{fig:child}
\end{figure}

\begin{figure}[!h]
    \centering
    \includegraphics[width=0.8\linewidth]{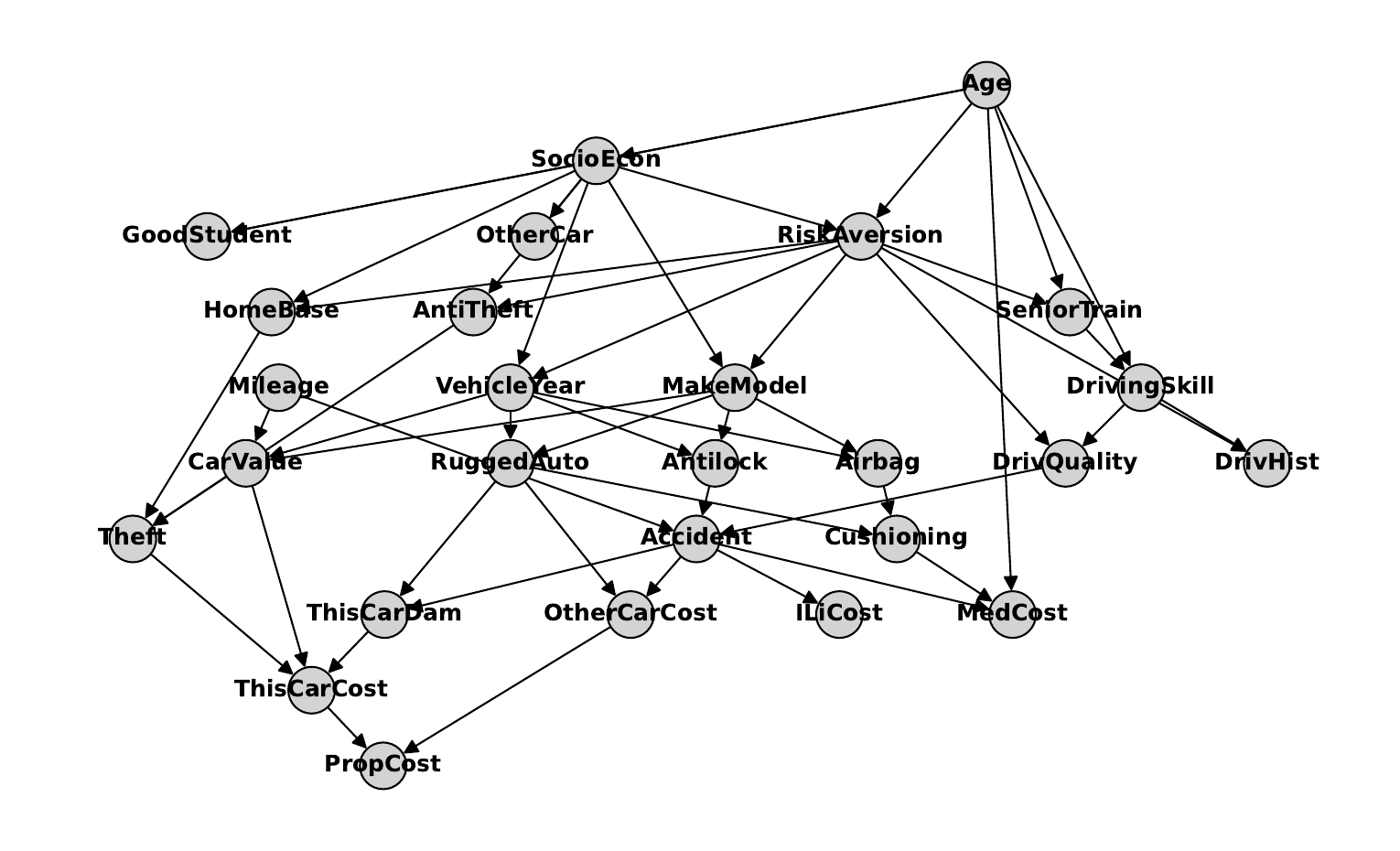}
    \caption{Ground truth Causal Graph for insurance data.}
    \label{fig:insurance}
\end{figure}

\begin{figure}[!h]
    \centering
    \includegraphics[width=0.75\linewidth]{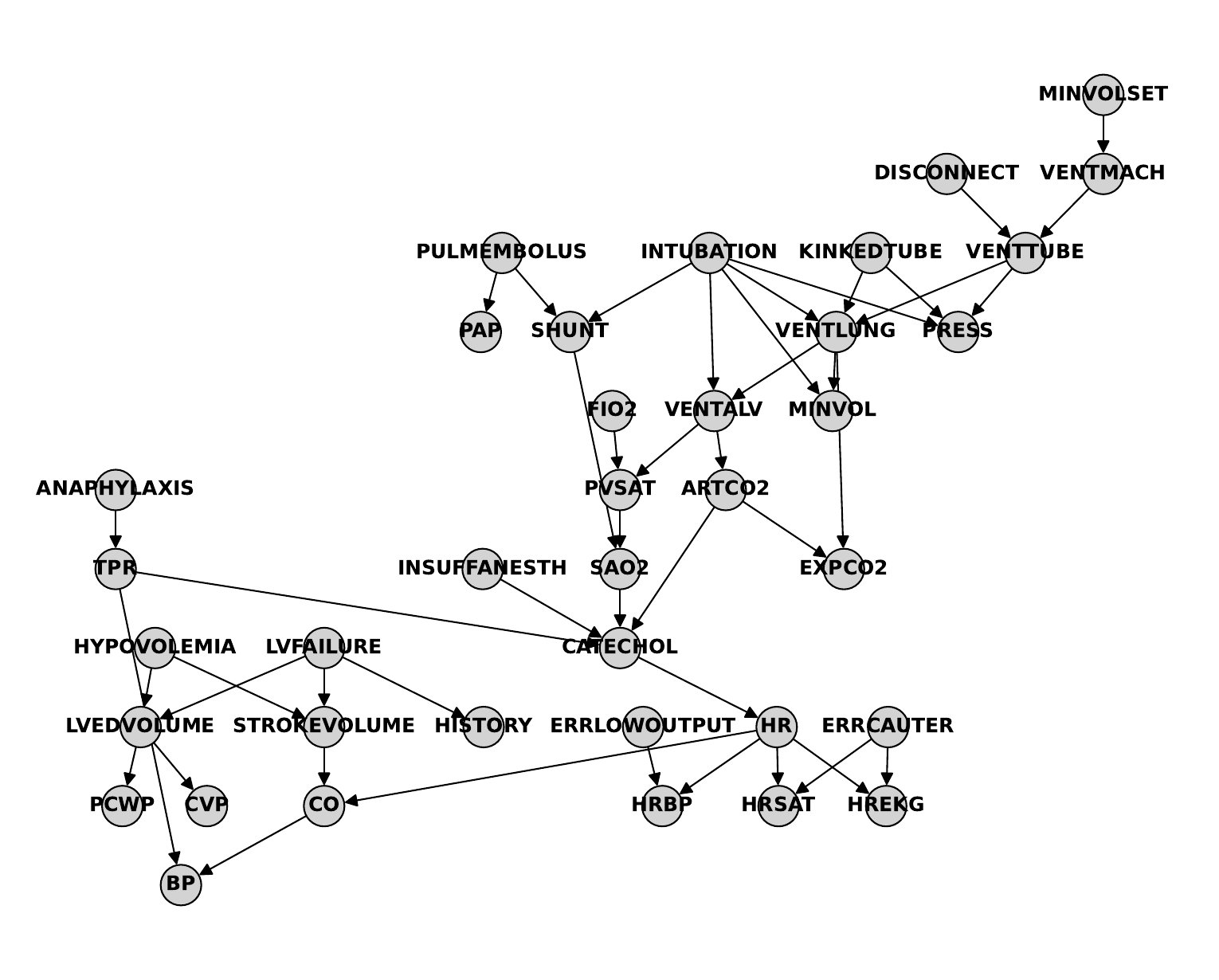}
    \caption{Ground truth Causal Graph for alarm data.}
    \label{fig:alarm}
\end{figure}


\begin{figure}[!h]
    \centering
    \setcounter{subfigure}{0} 
    \subfigure[Out degree distribution of Asia data.]{
        \centering
        \includegraphics[width=0.45\textwidth]{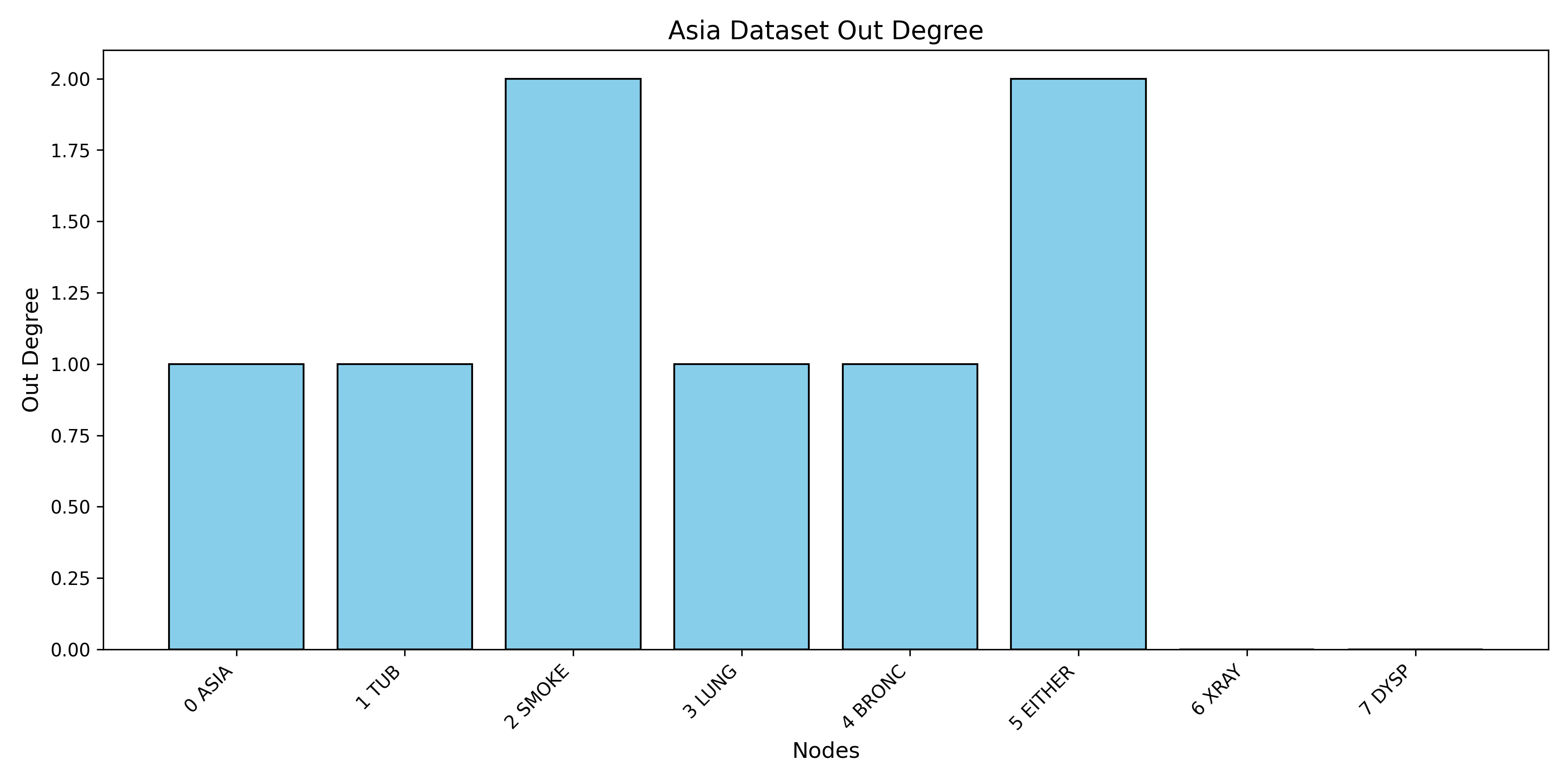}
        \label{fig:asia_out}
    }
    \subfigure[Out degree distribution of Child data.]{
        \centering
        \includegraphics[width=0.45\textwidth]{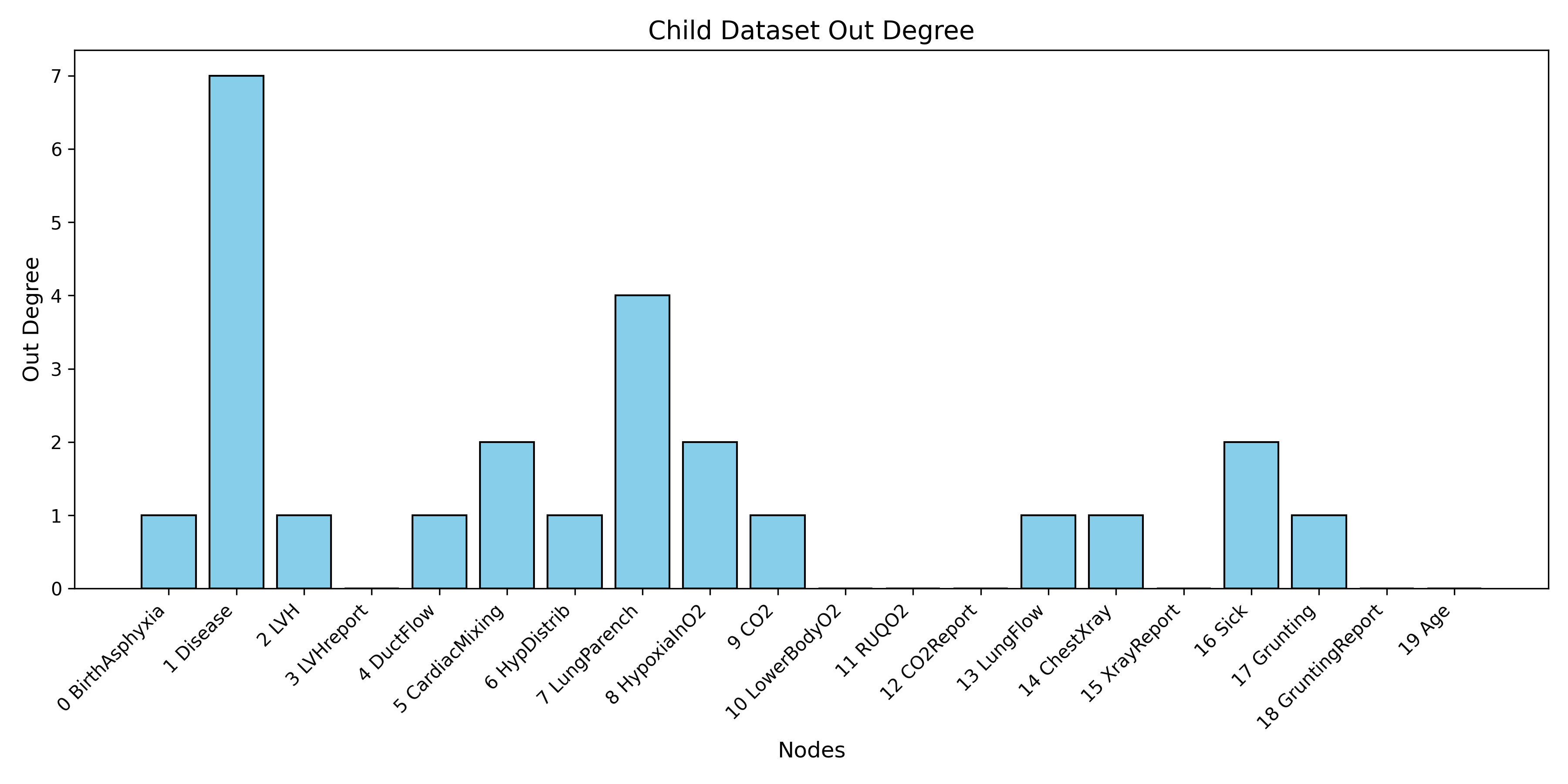}
        \label{fig:child_out}
    }
    \centering
    \subfigure[Out degree distribution of Insurance data.]{
        \centering
        \includegraphics[width=0.45\textwidth]{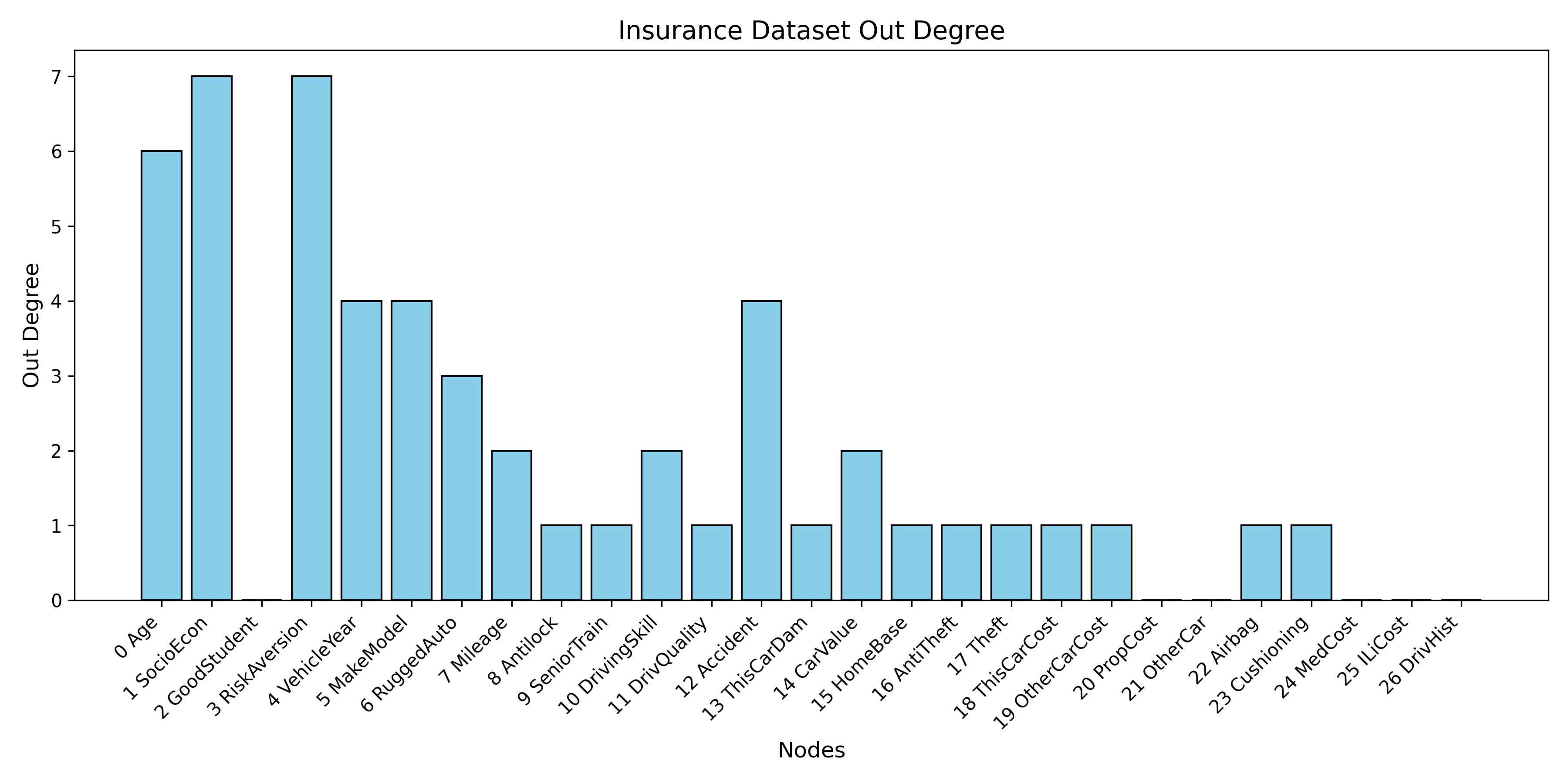}
        \label{fig:insurance_out}
    }
    \subfigure[Out degree distribution of Alarm data.]{
        \centering
        \includegraphics[width=0.45\textwidth]{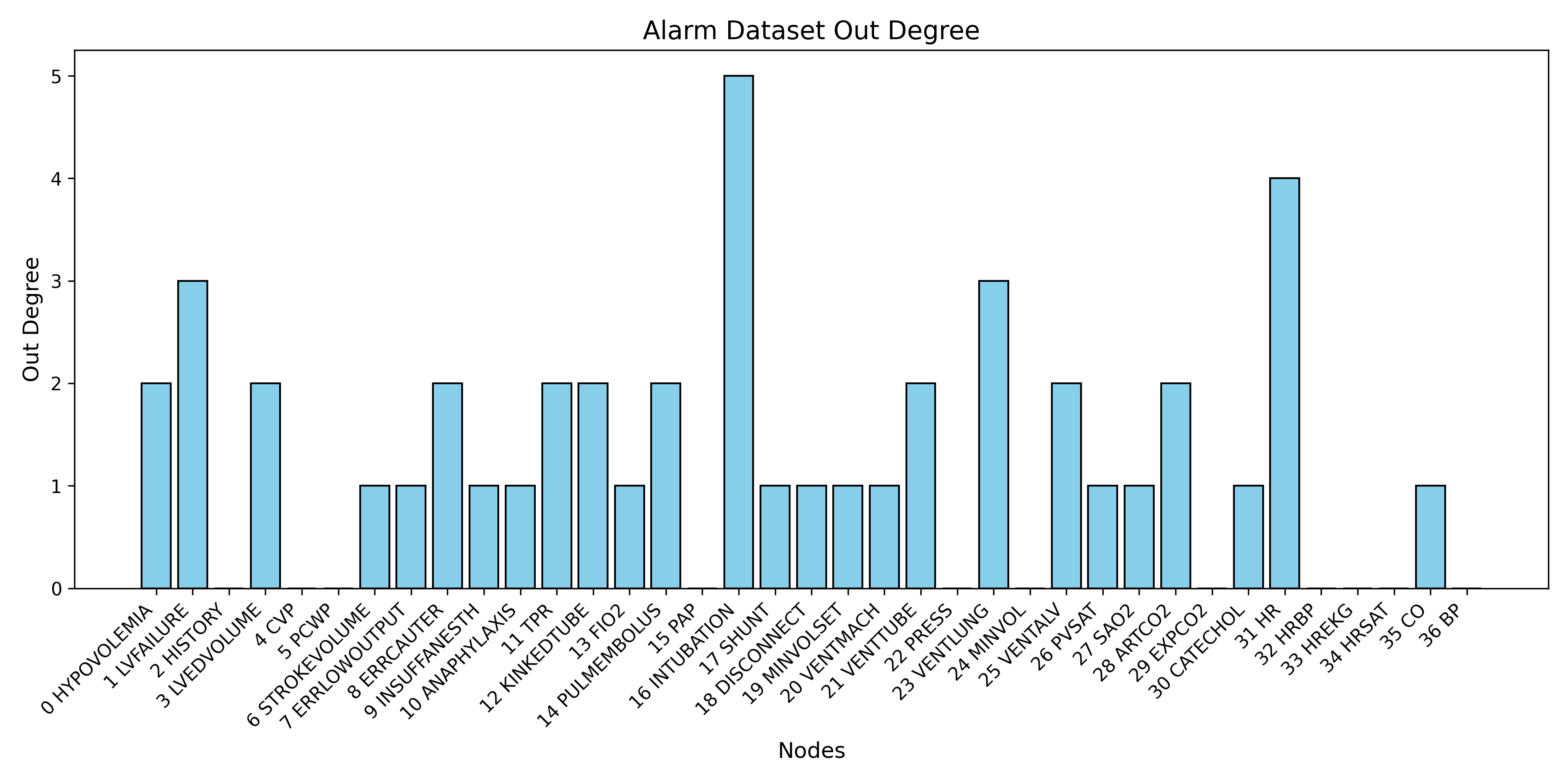}
        \label{fig:alarm_out}
    }
    \caption{Out Degree of 4 Dataset.}
    \label{fig:graph_out}
\end{figure}






\section{More Details of Experiments}
\label{appdx:exp}
\subsection{ENCO Hyperparameters}
For experiments using the ENCO framework, we used the exact parameters reported by ~\citep{lippe2022enco}. These parameters are provided in Table~\ref{table:hparams-enco} to ensure the completeness of our report.

\begin{table}[ht]
    \caption{Hyperparameters used for the ENCO framework.}
    \centering
    \centering\small\sc
\resizebox{\linewidth}{!}{
    \begin{tabular}{ lc }
        \toprule
        parameter                    & value \\
        \midrule
        Sparsity regularizer $\lambda_{sparse}$                &  $\num{4e-3}$   \\
        Distribution model                &  2 layers, hidden size 64, LeakyReLU($\alpha = 0.1$)   \\
        Batch size & 128 \\
        Learning rate - model & $\num{5e-3}$\\
        Weight decay - model & $\num{1e-4}$\\
        Distribution fitting iterations F  & $\num{1000}$\\
        Graph fitting iterations G & $\num100$\\
        Graph samples K & 100 \\
        Epochs & 30 \\
        Learning rate - $\gamma$ & $\num{2e-2}$ \\
        Learning rate - $\theta$ & $\num{1e-1}$\\
        \bottomrule
    \end{tabular}}
    \label{table:hparams-enco}
\end{table}



\subsection{Detailed Metrics}

In this section, we present the details of 3 different metrics mentioned in the experiment part.

\begin{itemize}
    \item The Structural Hamming Distance (SHD)~\citep{shd}: SHD is a widely used metric for comparing graphs based on their adjacency matrices. It measures the number of discrepancies between two binary adjacency matrices, counting the minimum number of edges that need to be removed $D$, added $A$, or reversed $R$ in order to obtain the truth causal graph.
    $$
    SHD = A+D+R
    $$
    \item Structural Intervention Distance (SID)~\citep{peters2015structural}: SID evaluates the similarity between DAGs$(\mathcal{G}, \mathcal{H})$ based on their ability to capture causal effects. Specifically, it quantifies the number of incorrectly inferred intervention distributions, highlighting how false edges in the generated graph can impact the derived causal effects.
\begin{align*}
\text{SID} &= 
\# \{ (i, j), i \neq j \mid \text{the intervention distribution from}\\
& i \text{ to } j \  \text{is falsely estimated by } \mathcal{H} \text{ with respect to } \mathcal{G} \}.
\end{align*}

\item Balanced scoring function(BSF)~\citep{constantinou2019evaluating}: BSF eliminates bias in assessing structure learning algorithms by adjusting penalties and rewards based on the occurrence rates of dependencies and independencies in the ground truth graph.

\[
\text{BSF} = \frac{1}{2} \left( \frac{\text{TP}}{a} + \frac{\text{TN}}{i} - \frac{\text{FP}}{i} - \frac{\text{FN}}{a} \right),
\]
where $TP$ represents the true positives, $TN$ the true negatives, $FP$ the false positives, and $FN$ the false negatives. The variable $a$ denotes the number of arcs in the true graph, while $i = \frac{|N| \times (|N| - 1)}{2} - a$ represents the number of independencies in the true graph, where $|N|$ is the total number of nodes.

\end{itemize}

\subsection{Final Causal Graph}

In this section, we present the final causal graph after $T = 33$, total sample $N = 1056$ results with GIT, Human, and \ours.

\begin{figure}[!htb]
    \centering
    \includegraphics[width=\linewidth]{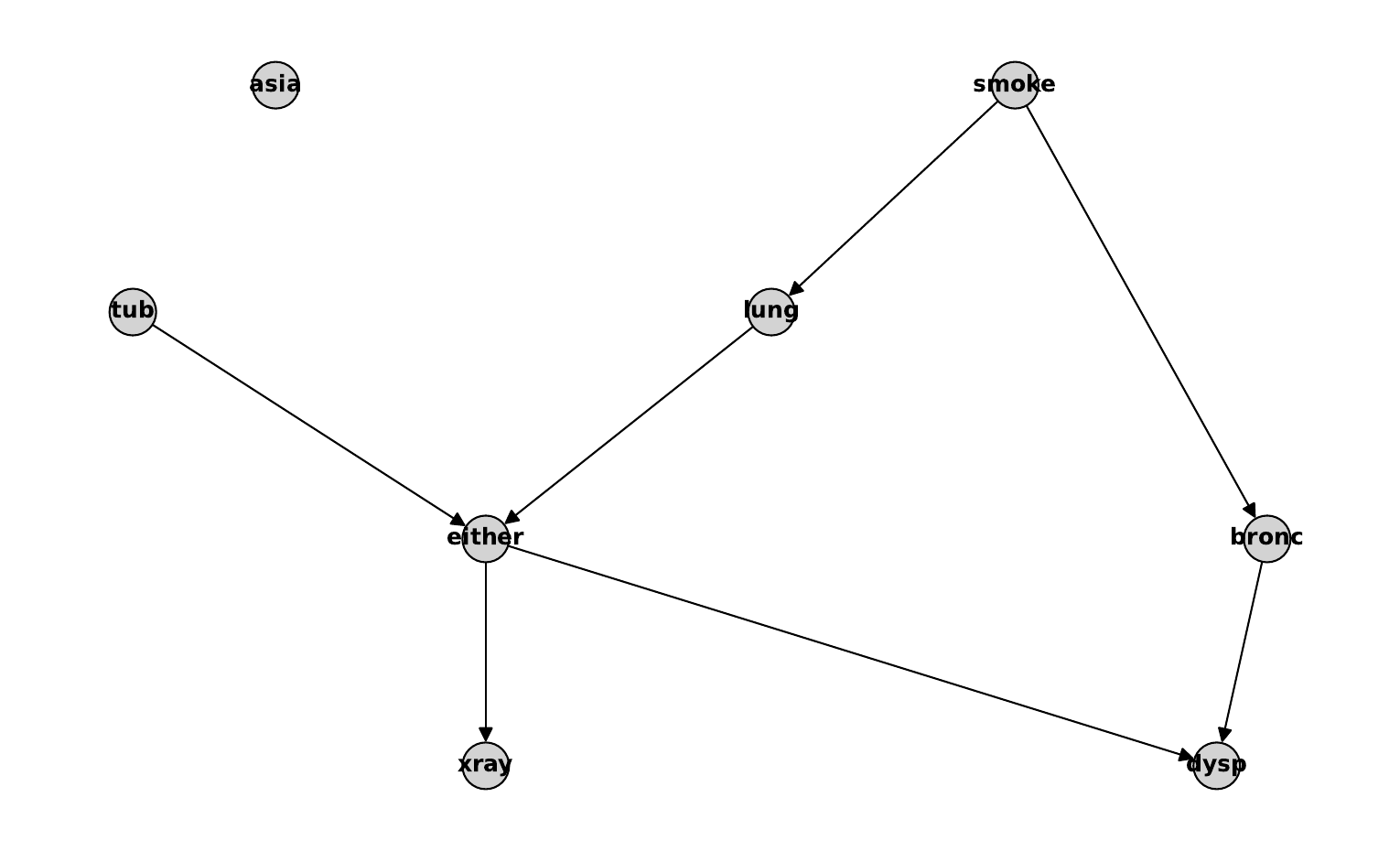}
    \caption{\ours final causal graph for asia dataset}
    \label{fig:asia_ours}
\end{figure}

\begin{figure}[!htb]
    \centering
    \includegraphics[width=\linewidth]{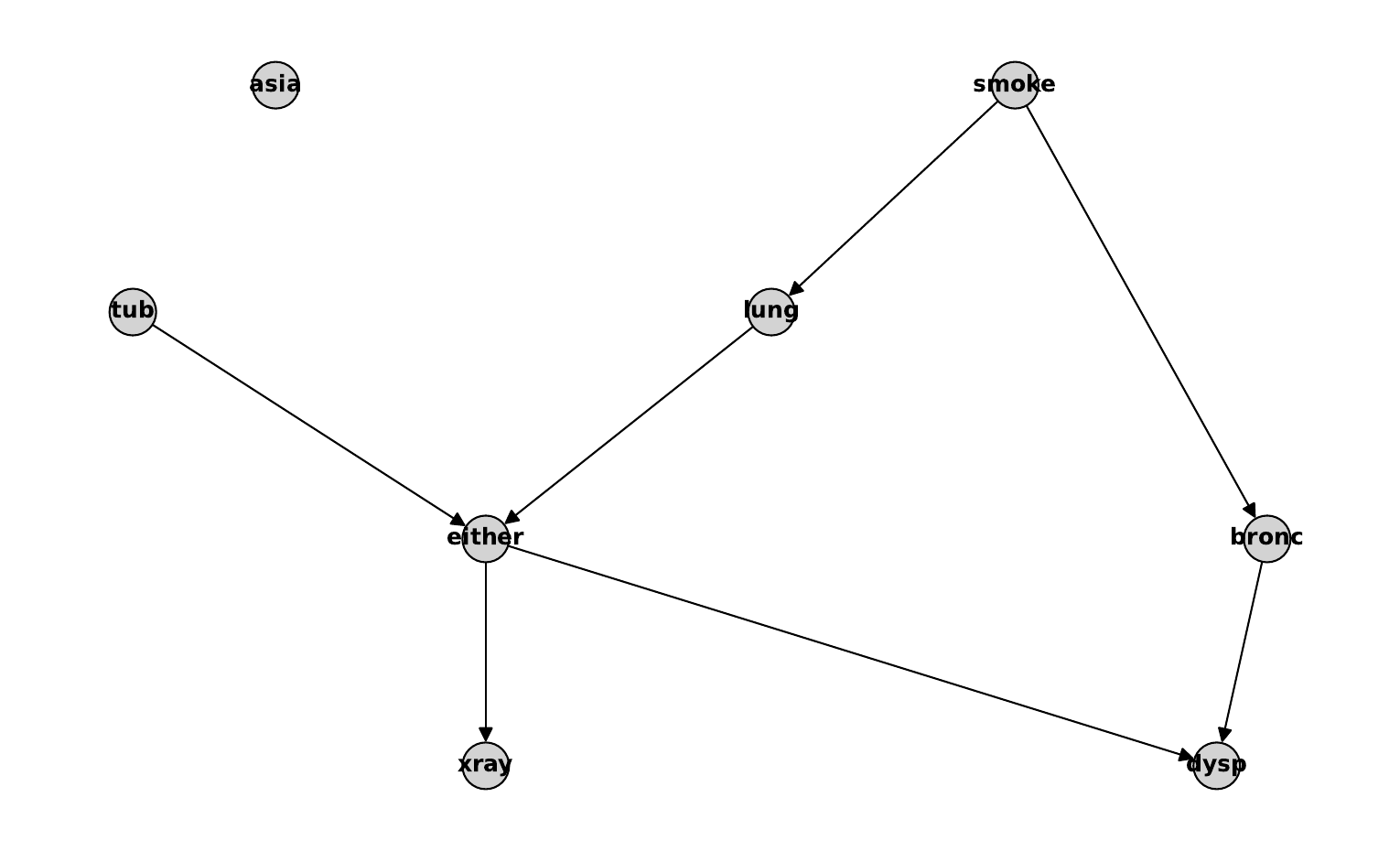}
    \caption{GIT's final causal graph for asia dataset}
    \label{fig:asia_git}
\end{figure}

\begin{figure}[!htb]
    \centering
    \includegraphics[width=\linewidth]{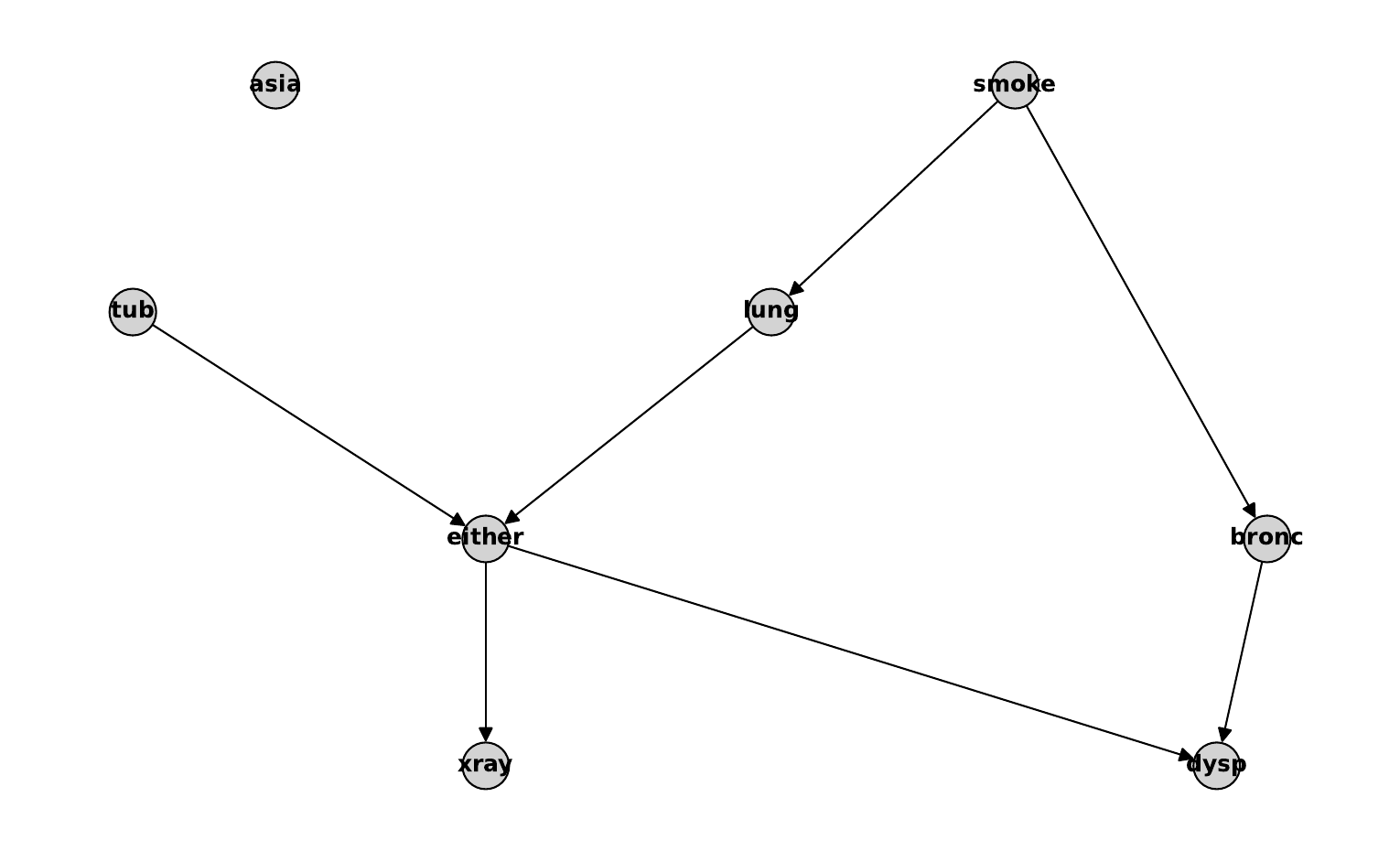}
    \caption{Huamn's final causal graph for asia dataset}
    \label{fig:asia_human}
\end{figure}

\begin{figure}[!htb]
    \centering
    \includegraphics[width=\linewidth]{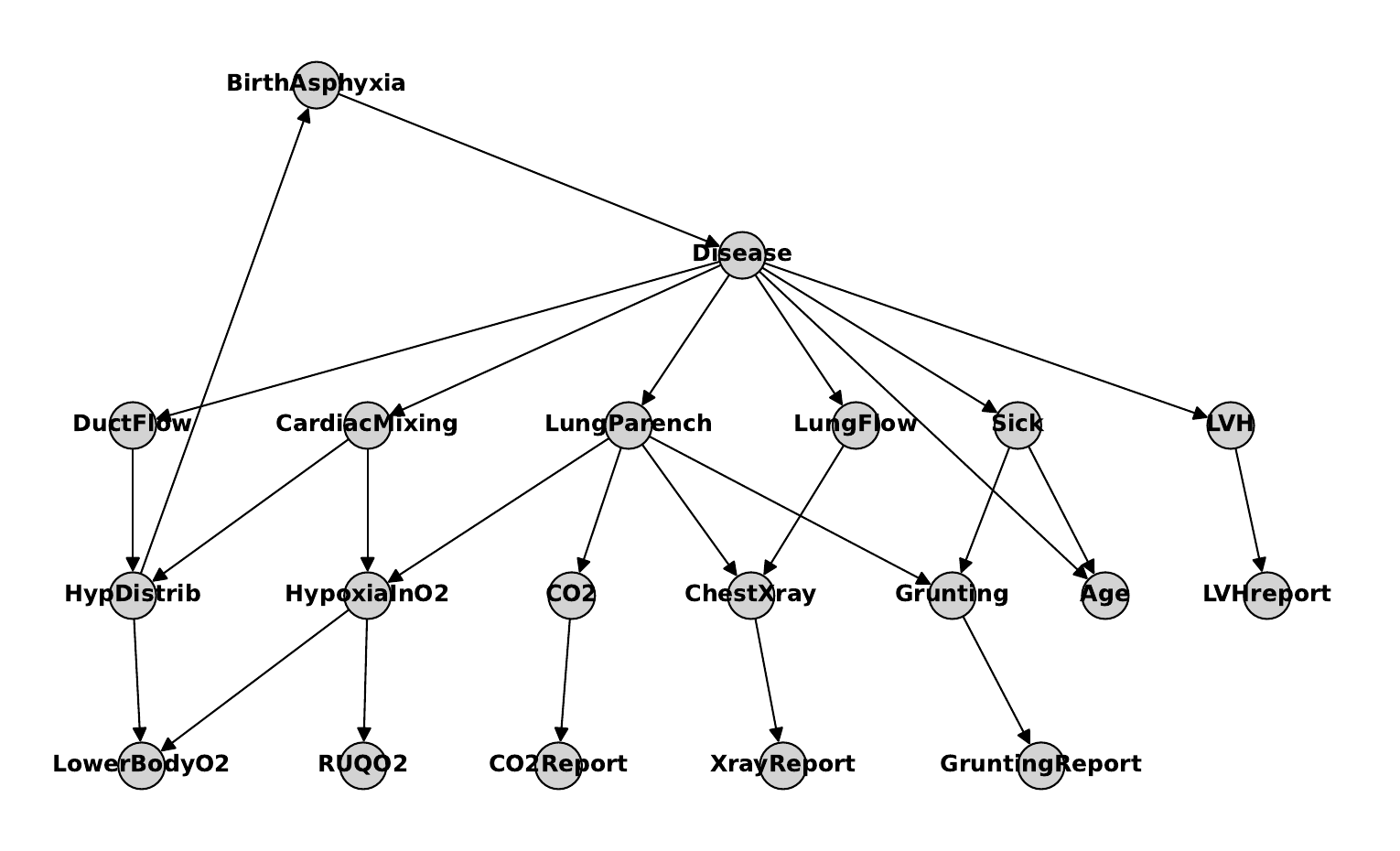}
    \caption{\ours final causal graph for child dataset}
    \label{fig:child_ours}
\end{figure}

\begin{figure}[!htb]
    \centering
    \includegraphics[width=\linewidth]{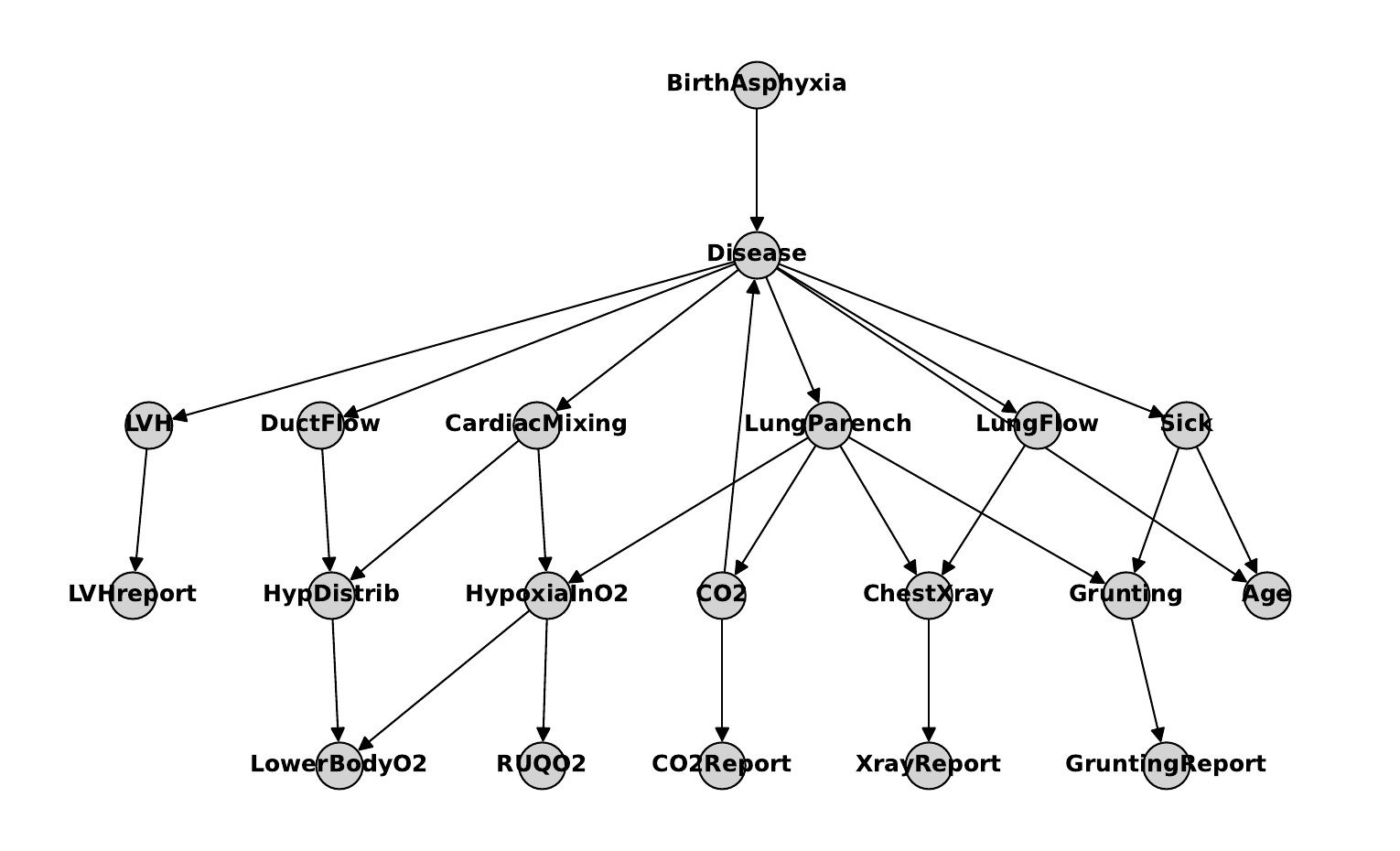}
    \caption{Human's final causal graph for child dataset}
    \label{fig:child_human}
\end{figure}

\begin{figure}[!htb]
    \centering
    \includegraphics[width=\linewidth]{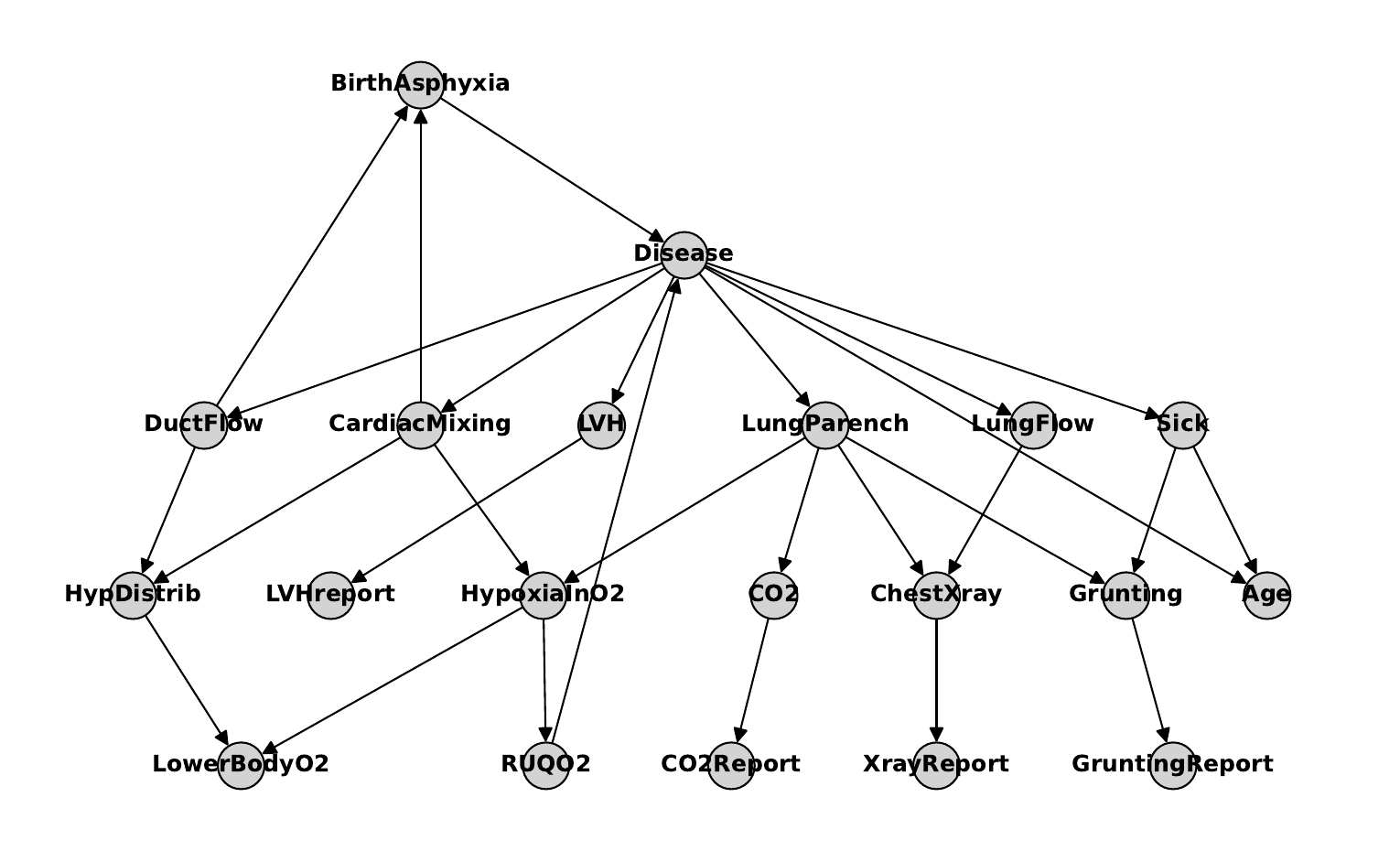}
    \caption{GIT's final causal graph for child dataset}
    \label{fig:child_git}
\end{figure}

\begin{figure}[h]
    \centering
    \includegraphics[width=\linewidth]{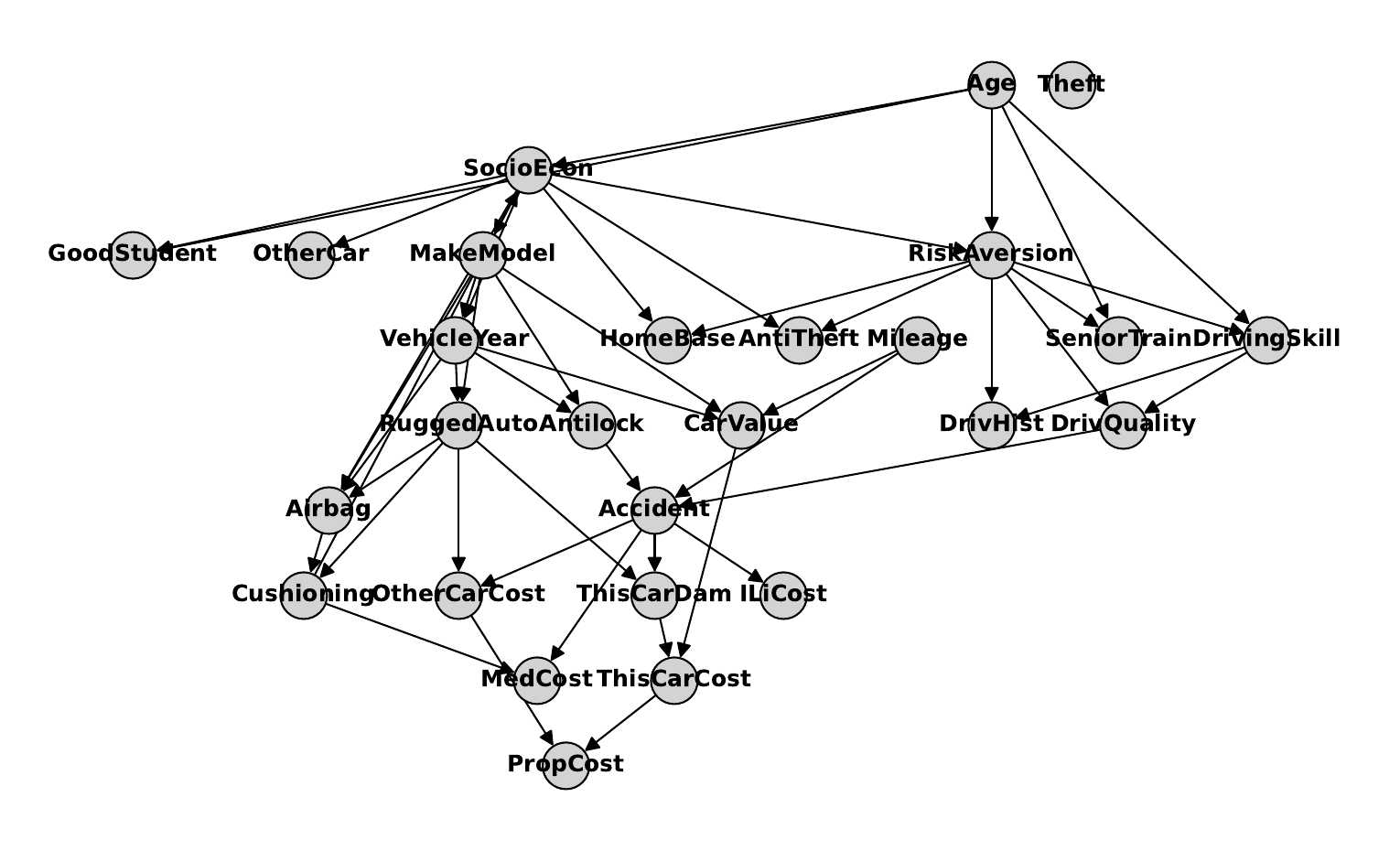}
    \caption{\ours final causal graph for insurance dataset}
    \label{fig:insurance_ours}
\end{figure}

\begin{figure}[!htb]
    \centering
    \includegraphics[width=\linewidth]{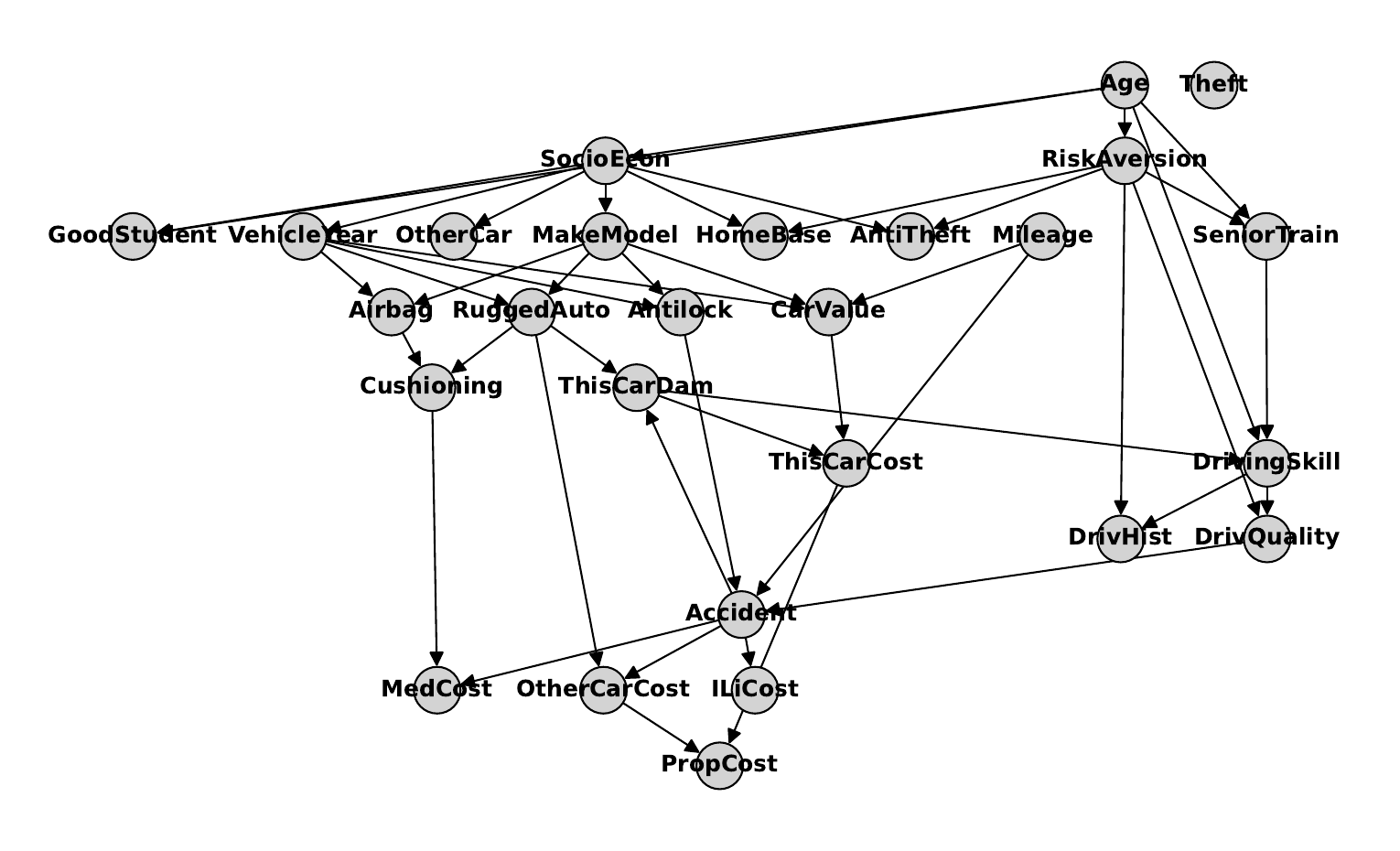}
    \caption{Human's final causal graph for insurance dataset}
    \label{fig:insurance_human}
\end{figure}

\begin{figure}[!htb]
    \centering
    \includegraphics[width=\linewidth]{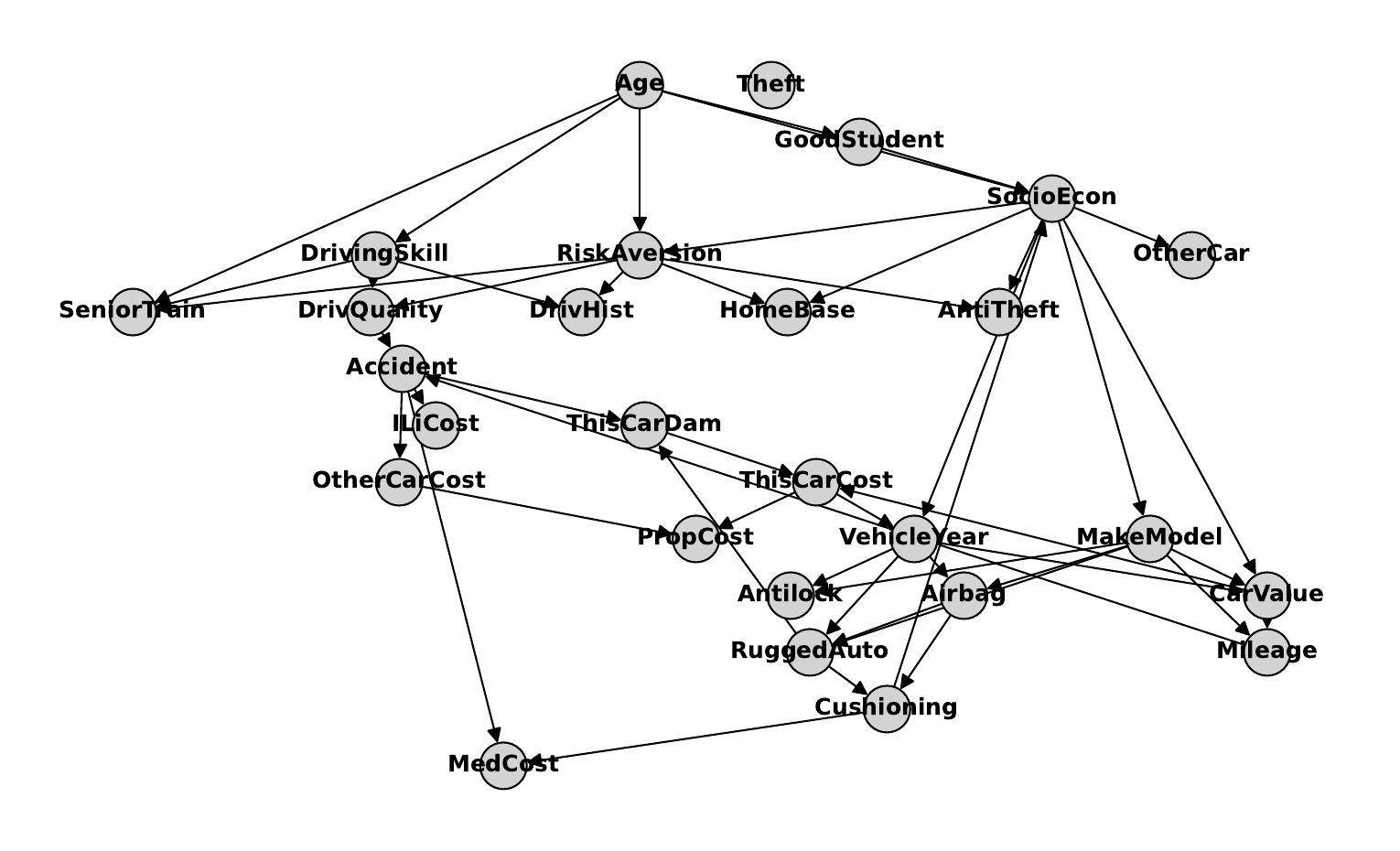}
    \caption{GIT's final causal graph for insurance dataset}
    \label{fig:insurance_git}
\end{figure}

\begin{figure}[!htb]
    \centering
    \includegraphics[width=\linewidth]{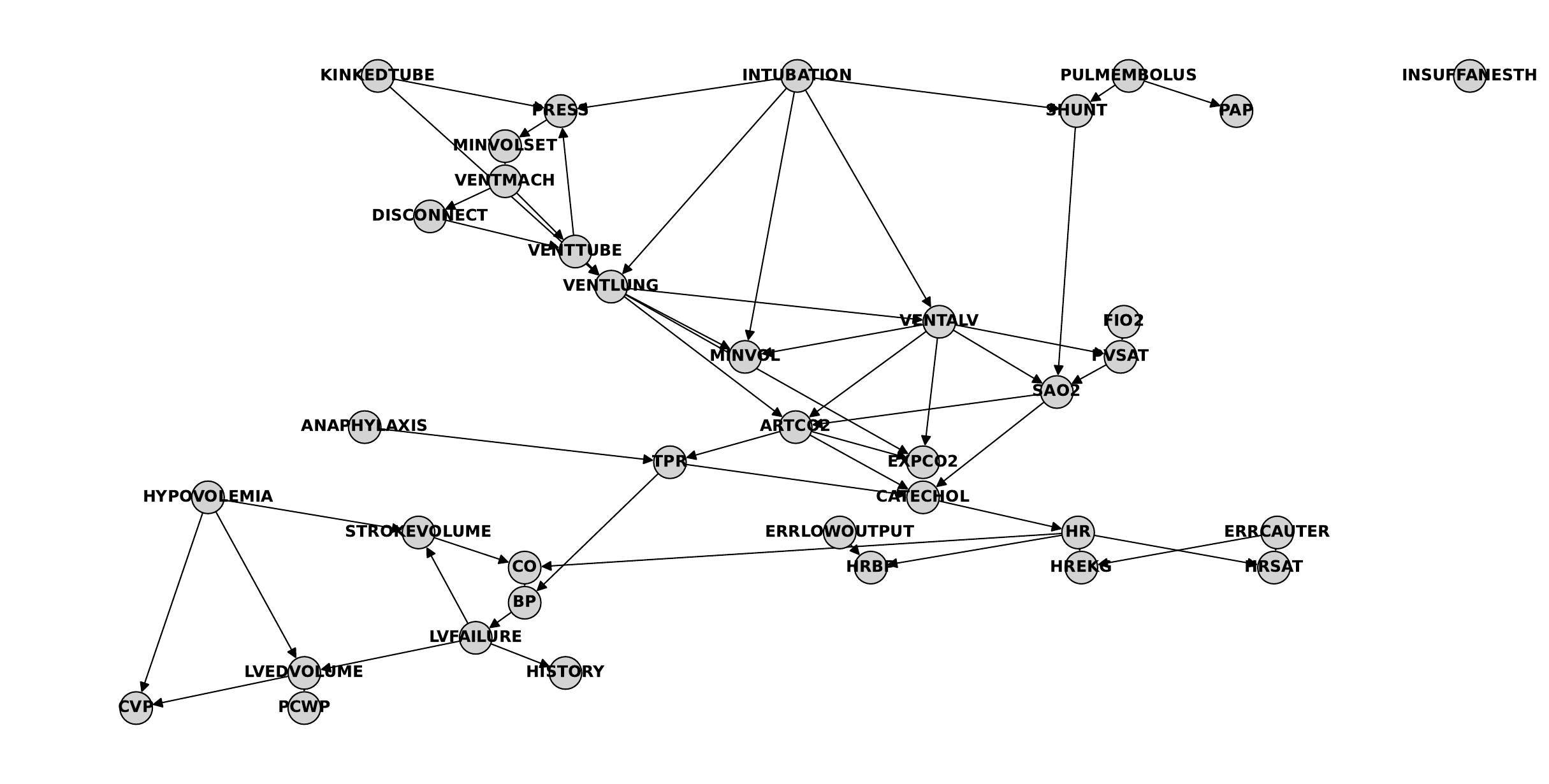}
    \caption{\ours final causal graph for alarm dataset}
    \label{fig:alarm_ours}
\end{figure}

\begin{figure}[!htb]
    \centering
    \includegraphics[width=\linewidth]{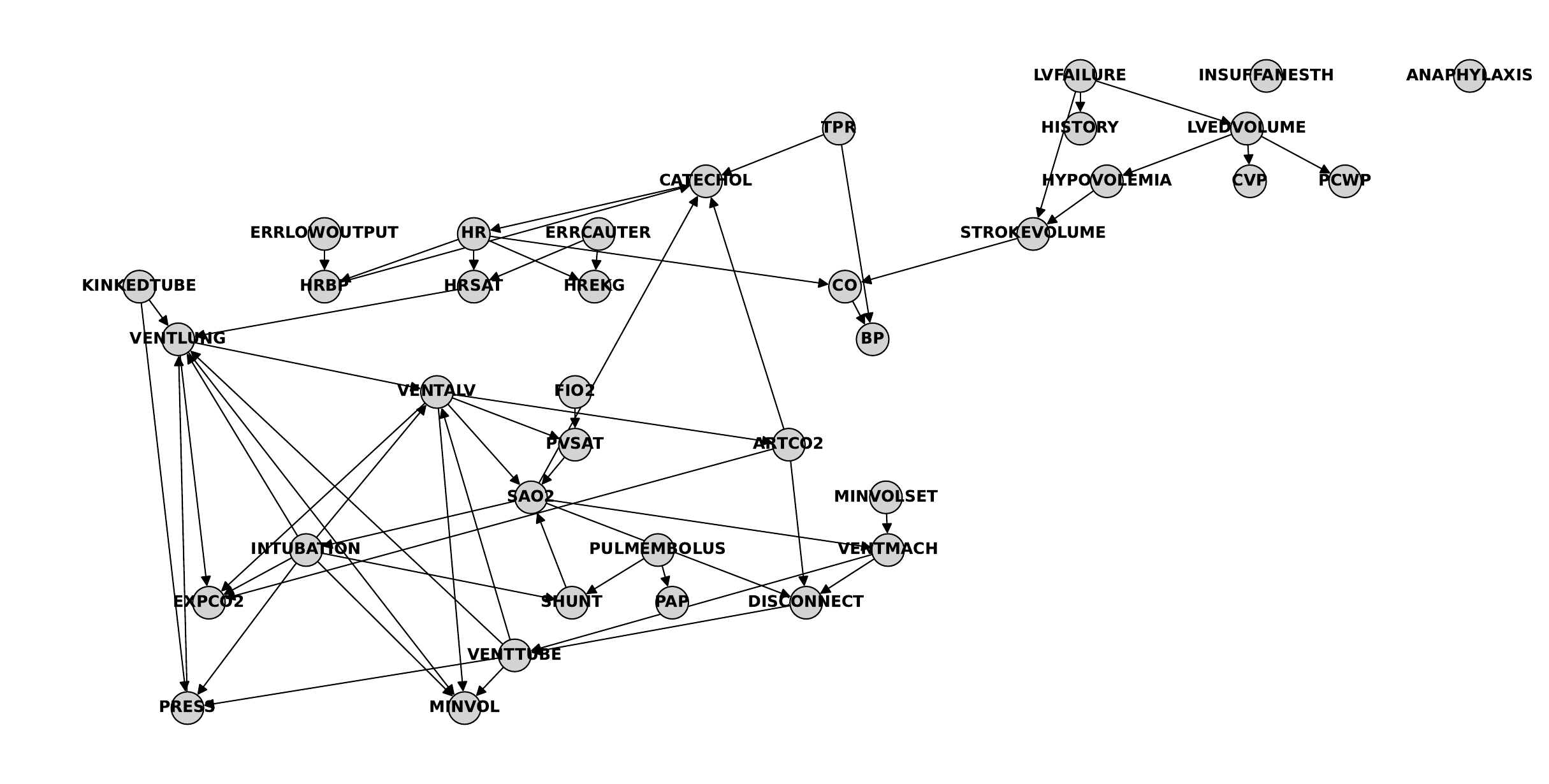}
    \caption{Human's final causal graph for alarm dataset}
    \label{fig:alarm_humans}
\end{figure}

\begin{figure}[!htb]
    \centering
    \includegraphics[width=\linewidth]{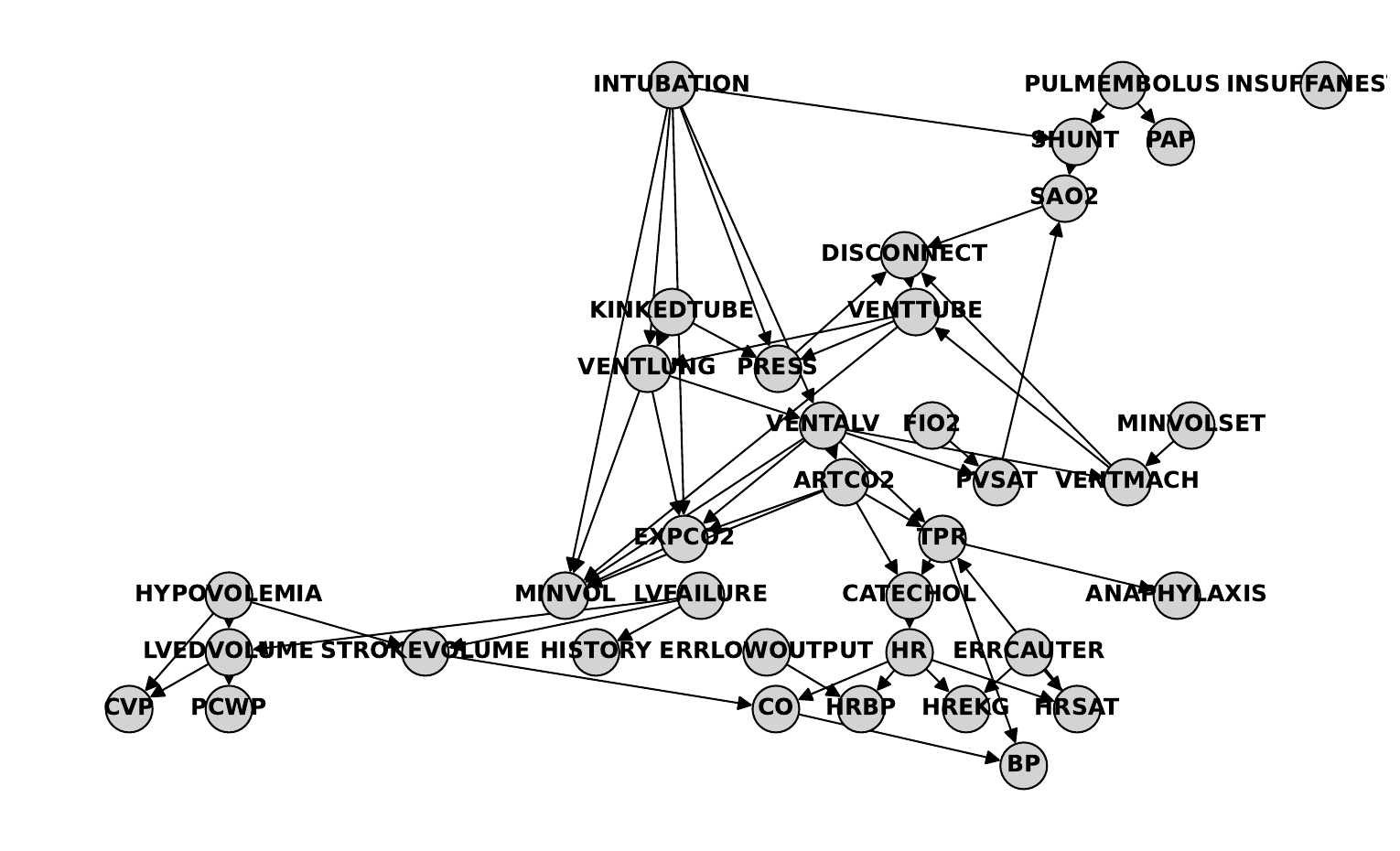}
    \caption{GIT's final causal graph for alarm dataset}
    \label{fig:alarm_git}
\end{figure}

\clearpage

\section{Convergence of causal discovery with \ours}

In this section, we provide a convergence argument for \ours, 
which combines a Large Language Model (LLM) warmup phase with a numerical-based intervention targeting strategy (e.g., GIT \citep{olko2023git}).

\subsection{Preliminaries and Notation}

\paragraph{Structural Causal Models and Online Causal Discovery.}
We use the same definition of Structural Causal Models (SCMs), directed acyclic graphs (DAGs), and single-node interventions (hard interventions) described in, e.g.,~\citep{lippe2022enco}. Suppose we have:
\[
G^* = (V, E^*), 
\]
where $V = \{1,2,\dots,n\}$ indexes the causal variables $(X_1, \dots, X_n)$, and we wish to recover $G^*$. In an online setting, at each round $t = 1,2,\dots,T$, we choose an intervention target $i_t \in V$ and obtain a small batch of interventional samples from $X$ under that intervention. This new interventional data is then used to update our current belief about the causal structure and functional parameters.

\paragraph{LeGIT and the Warmup Stage.}
LeGIT begins with a small number of \emph{warmup} rounds, $T_\text{warmup}$, possibly augmented by an additional \emph{bootstrap} stage $T_\text{bootstrapped}$. Across these initial stages, an LLM proposes intervention targets based on domain-specific descriptions or meta-information about the variables. Once the warmup phases are finished, the algorithm reverts to a purely numerical-based strategy for intervention selection---for example, GIT \citep{olko2023git} or a Bayesian method~\citep{Brouillard2020DifferentiableCD}. Let $I_t$ be the random variable denoting the chosen intervention target at round $t$; then

\[
I_t = 
\begin{cases}
\text{LLM-based selection},\\
\ \ \ \ \ \ \ \ \ \ \ \ \ \ \ \ \ \ \ \ \ \  \ \ \ \  t \le 2*(T_\text{warmup}+T_\text{bootstrapped}),\\
\text{numerical-based selection},\\
\ \ \ \ \ \ \ \ \ \ \ \ \ \ \ \ \ \ \ \ \ \ \ \ \  2*(T_\text{warmup}+T_\text{bootstrapped}) <t \leq T.
\end{cases}
\]

\subsection{Convergence Proof}

Throughout this section, we make the following standard assumptions:

\begin{itemize}
    \item \textbf{A1 (Faithfulness)}. The true distribution is \emph{faithful} to a unique causal DAG $G^*$. 
    \item \textbf{A2 (Sufficiency)}. There are no hidden confounders, i.e., all relevant causal variables are observed.
    \item \textbf{A3 (Convergent Base-Algorithm)}. After collecting a sufficient number of correct interventional data points, the base numerical method (e.g., GIT + ENCO) converges to $G^*$; see \citep{olko2023git,lippe2022enco} for formal statements.
\end{itemize}

Our main theorem shows that, under A1--A3, and given that the LLM-based warmup selects meaningful root-cause or high-influence variables in at least some fraction of the initial rounds, LeGIT converges to the correct graph $G^*$ in the limit of acquiring more interventional data.

\subsection{Key Lemma: Warmup Rounds Provide Informative Interventions}

Let $\phi_{\text{LLM}}$ be the set of selected nodes from the LLM-based warmup. Although LLM selection can be imperfect, we show that with nonzero probability, $\phi_{\text{LLM}}$ contains enough influential/parent nodes to break symmetries or ambiguities in the MEC, thereby accelerating or guaranteeing eventual convergence.

\begin{lemma}[LLM Warmup is Sufficiently Informative]
\label{lemma:Information_Guarantee}
Suppose that, during the warmup stage, the LLM selects a node $v$ which is a direct cause (or ancestor) of at least one child $c$ that is currently ambiguous or misoriented in the model. Intervening on $v$ yields significant new information about the structure among $\{v, c\}$. If the warmup stage includes such interventions on enough distinct parents or high-degree nodes, the post-warmup structure has strictly fewer edges inconsistent with $G^*$ on average. 
\end{lemma}

\begin{proof}
When the LLM intervenes on $v$, we obtain data from $P(X \mid do\{v\})$. Under A1--A2, the subsequent updates to the structural parameters will, with high probability, remove incorrect edges around $v$ or correct orientation errors. Repeating this for a sufficient set of $v$ nodes with nontrivial out-degrees ensures that many orientation and adjacency ambiguities in $G$ are resolved. A formal statement follows directly from standard identifiability arguments of single-node interventions. 
\end{proof}

\subsection{Convergence Argument for the Combined Procedure}

After $2(T_{\text{warmup}} + T_{\text{bootstrapped}})$ rounds, the base method (e.g., GIT) takes over and selects all subsequent targets $\{I_t\}_{t >2 ( T_\text{warmup}+T_\text{bootstrapped})}$. The following proposition states that if the base method itself converges when given sufficiently many informative interventions (as assumed by A3), then the combination of warmup + base method must also converge.

\begin{proposition}[Convergence of LeGIT]
\label{prop:Convergence}
Assume the LLM-based warmup stage provides a non-empty set of interventions that reduce critical ambiguities in the causal structure (Lemma~\ref{lemma:Information_Guarantee}). Let the base method be any procedure that is guaranteed to converge to $G^*$ if it acquires sufficiently many samples from the relevant parts of the DAG (A3). Then, as $T \to \infty$, \emph{LeGIT} converges to the correct causal DAG $G^*$ with probability 1.
\end{proposition}

\begin{proof}[Proof]
By Lemma~\ref{lemma:Information_Guarantee}, after the LLM warmup stage, the posterior space of graphs is already closer to the true DAG. That is, the number of structural ambiguities or misoriented edges around key high-influence nodes is reduced or completely resolved.

From round $t = 2(T_{\text{warmup}} + T_{\text{bootstrapped}}) + 1$ onward, the intervention targeting is dictated by the base method (e.g., GIT). Under assumption~A3, we know that if GIT (plus the underlying gradient-based causal discovery algorithm like ENCO) is run on a system with sufficiently many informative interventions, it converges to $G^*$. Because the warmup stage of LeGIT has by design intervened on crucial nodes to reduce ambiguities, the base method from that point sees a significantly less confounded or ambiguous search space and, with high probability, chooses further interventions that refine the partial solution until it converges to $G^*$. 

Thus, by the properties of the base method's convergence proof (see \citep{olko2023git}), the entire procedure (LLM warmup + GIT) converges to $G^*$ given sufficient total rounds $T$.
\end{proof}

\begin{remark}[Extension to Other Methods]
Although we have discussed GIT and ENCO as an illustrative example, any gradient-based or Bayesian-based method that ensures correct discovery given a suitable variety of interventions can replace GIT in Step~(b) of LeGIT. Under the same conditions (A1--A3), the combined procedure likewise converges to the true DAG $G^*$.
\end{remark}

\section{Examples of Prompts}
\label{appdx:prompt}

For robust performance, we actually shuffle the order of variable descriptions following the self-consistency prompt skill. We provide the prompt templates and the description of the variables used in \ours below.

\begin{tcolorbox}[colback=blue!5!white,colframe=blue!75!black,title=Asia Warmup Prompt,breakable,
    listing only,
     listing options={
        breaklines=true,
        basicstyle=\ttfamily\small,
        columns=fullflexible}]
You are a helpful assistant and expert in lung disease research. Here are some tips that you can pay attention to:

1. Assess whether there is a direct causal relationship, and consider potential confounding variables that might affect the relationship that could potentially not causal relationship.

2. Distinguish between correlations and causation; verify that correlations are not mistaken for causal relationships.

3. Ensure the correct temporal order of variables; confirm that the cause precedes the effect.

Assuming we can do interventions to all the variables, your job is to assist in designing the best intervention experiments among the following variables to help discover their causal relations:

<dysp>: whether or not the patient has dyspnoea, also known as shortness of breath

<smoke>: whether or not the patient is a smoker

<xray>: whether or not the patient has had a positive chest xray

<lung>: whether or not the patient has lung cancer

<tub>: whether or not the patient has tuberculosis

<asia>: whether or not the patient has recently visited asia

<either>: whether or not the patient has either tuberculosis or lung cancer

<bronc>: whether or not the patient has bronchitis

Assuming we can do interventions to all the variables, given the aforementioned variables and their descriptions, can you **echo your knowledge those variables**, **temporally analyze** their relations, and then **choose the best 4 intervention targets from all the variables** which hopefully are the root causes of the other variables to start our analysis of their causal relations?

Let's think and analyze step by step. Then, provide your final answer (variable names only) within the tags <Answer>...</Answer>, separated by ", ".
\end{tcolorbox}

\begin{tcolorbox}[colback=blue!5!white,colframe=blue!75!black,title=Child Warmup Prompt, breakable,
    listing only,
    listing options={
        breaklines=true,
        basicstyle=\ttfamily\small,
        columns=fullflexible}]
You are a helpful assistant and expert in children's disease research. Here are some tips that you can pay attention to:

1. Assess whether there is a direct causal relationship, and consider potential confounding variables that might affect the relationship that could potentially not causal relationship.

2. Distinguish between correlations and causation; verify that correlations are not mistaken for causal relationships.

3. Ensure the correct temporal order of variables; confirm that the cause precedes the effect.

Assuming we can do interventions to all the variables, your job is to assist in designing the best intervention experiments among the following variables to help discover their causal relations:

<LungFlow>: low blood flow in the lungs

<ChestXray>: having a chest x-ray

<Disease>: infant methemoglobinemia

<Grunting>: grunting in infants

<Age>: age of infant at disease presentation

<XrayReport>: lung excessively filled with blood

<RUQO2>: level of oxygen in the right upper quadriceps muscle

<DuctFlow>: blood flow across the ductus arteriosus

<HypoxiaInO2>: hypoxia when breathing oxygen

<Sick>: presence of an illness

<CO2Report>: a document reporting high level of CO2 levels in blood

<LungParench>: the state of the blood vessels in the lungs

<LVH>: having left ventricular hypertrophy

<LowerBodyO2>: level of oxygen in the lower body

<BirthAsphyxia>: lack of oxygen to the blood during the infant's birth

<CO2>: level of CO2 in the body
<LVHreport>: report of having left ventri

<GruntingReport>: report of infant grunting

<CardiacMixing>: mixing of oxygenated and deoxygenated blood

<HypDistrib>: low oxygen areas equally distributed around the body

Assuming we can do interventions to all the variables, given the aforementioned variables and their descriptions, can you **echo your knowledge those variables**, **temporally analyze** their relations, and then **choose the best 4 intervention targets from all the variables** which hopefully are the root causes of the other variables to start our analysis of their causal relations?

Let's think and analyze step by step. Then, provide your final answer (variable names only) within the tags <Answer>...</Answer>, separated by ", ".
\end{tcolorbox}

\begin{tcolorbox}[colback=blue!5!white,colframe=blue!75!black,title=Insurance Warmup Prompt, breakable,
    listing only,
    listing options={
        breaklines=true,
        basicstyle=\ttfamily\small,
        columns=fullflexible}]
You are a helpful assistant and expert in car insurance risks research. Here are some tips that you can pay attention to:

1. Assess whether there is a direct causal relationship, and consider potential confounding variables that might affect the relationship that could potentially not causal relationship.

2. Distinguish between correlations and causation; verify that correlations are not mistaken for causal relationships.

3. Ensure the correct temporal order of variables; confirm that the cause precedes the effect.

Assuming we can do interventions to all the variables, your job is to assist in designing the best intervention experiments among the following variables to help discover their causal relations:

<ThisCarDam>: damage to the car

<MakeModel>: owning a sports car

<OtherCarCost>: cost of the other cars

<PropCost>: ratio of the cost for the two cars

<AntiTheft>: car has anti-theft

<DrivQuality>: driving quality

<DrivHist>: driving history

<MedCost>: cost of medical treatment

<Mileage>: how much mileage is on the car

<Antilock>: car has anti-lock

<CarValue>: value of the car

<Accident>: severity of the accident

<OtherCar>: being involved with other cars in the accident

<SeniorTrain>: received additional driving training

<ILiCost>: inspection cost

<SocioEcon>: socioeconomic status

<ThisCar>: costs for the insured car

<Theft>: theft occured in the car

<Age>: age

<RuggedAuto>: ruggedness of the car

<GoodStudent>: being a good student driver

<VehicleYear>: year of vehicle

<HomeBase>: neighbourhood type

<ThisCarCost>: costs for the insured car

<Cushioning>: quality of cushinoning in car

<RiskAversion>: being risk averse

<DrivingSkill>: driving skill

<Airbag>: car has an airbad

Assuming we can do interventions to all the variables, given the aforementioned variables and their descriptions, can you **echo your knowledge those variables**, **temporally analyze** their relations, and then **choose the best 5 intervention targets from all the variables** which hopefully are the root causes of the other variables to start our analysis of their causal relations?

Let's think and analyze step by step. Then, provide your final answer (variable names only) within the tags <Answer>...</Answer>, separated by ", ".
\end{tcolorbox}

\begin{tcolorbox}[colback=blue!5!white,colframe=blue!75!black,title=Alarm Warmup Prompt,breakable,
    listing only,
    listing options={
        breaklines=true,
        basicstyle=\ttfamily\small,
        columns=fullflexible}
        ]
        
You are a helpful assistant and expert in alarm message system for patient monitoring system research. Here are some tips that you can pay attention to:

1. Assess whether there is a direct causal relationship, and consider potential confounding variables that might affect the relationship that could potentially not causal relationship.

2. Distinguish between correlations and causation; verify that correlations are not mistaken for causal relationships.

3. Ensure the correct temporal order of variables; confirm that the cause precedes the effect.

Assuming we can do interventions to all the variables, your job is to assist in designing the best intervention experiments among the following variables to help discover their causal relations:

<CATECHOL>: hormone made by the adrenal glands

<SAO2>: oxygen saturation of arterial blood

<VENTALV>: exchange of gas between the alveoli and the external environment

<ANAPHYLAXIS>: sever, life-threatening allergic reaction

<INSUFFANESTH>: whether there is insufficient anesthesia or not

<FIO2>: the concentration of oxygen in the gas mixture being inspired

<BP>: pressure of circulating blood against the walls of blood vessels

<PRESS>: breathing pressure

<VENTTUBE>: whether there is a breathing tube or not

<TPR>: amount of force exerted on circulating blood by vasculature of the body

<CO>: amount of blood pumped by the heart per minute

<PCWP>: pulmonary capillary wedge pressure

<ERRCAUTER>: whether there was an error during cautery or not

<KINKEDTUBE>: whether the chest tube is kinked or not

<PVSAT>: amount of oxygen bound to hemoglobin in the pulmonary artery

<INTUBATION>: process where a healthcare provider inserts a tube through a person's mouth or nose, then down into their trachea

<CVP>: measure of blood pressure in the vena cava

<HYPOVOLEMIA>: condition that occurs when your body loses fluid, like blood or water

<HRBP>: ratio of heart rate and blood pressure

<HREKG>: heart rate displayed on EKG monitor

<PAP>: blood pressure in the pulmonary artery

<EXPCO2>: expelled CO2

<ERRLOWOUTPUT>: error low output

<HISTORY>: previous medical history

<SHUNT>: hollow tube surgically placed in the brain (or occasionally in the spine) to help drain cerebrospinal fluid and redirect it to another location in the body where it can be reabsorbed

<VENTMACH>: the intensity level of a breathing machine

<VENTLUNG>: lung ventilation

<HRSAT>: measure of how much hemoglobin is currently bound to oxygen compared to how much hemoglobin remains unbound

<LVFAILURE>: occurs when there is dysfunction of the left ventricle causing insufficient delivery of blood to vital body organs

<DISCONNECT>: disconnection

<LVEDVOLUME>: amount of blood present in the left ventricle before contraction

<HR>: heart rate

<MINVOLSET>: the amount of time using a breathing machine

<PULMEMBOLUS>: sudden blockage in the pulmonary arteries, the blood vessels that send blood to your lungs

<STROKEVOLUME>: volume of blood pumped out of the left ventricle of the heart during each systolic cardiac contraction

<MINVOL>:  amount of gas inhaled or exhaled from a person's lungs in one minute

<ARTCO2>: arterial carbon dioxide

Assuming we can do interventions to all the variables, given the aforementioned variables and their descriptions, can you **echo your knowledge those variables**, **temporally analyze** their relations, and then **choose the best 4 intervention targets from all the variables** which hopefully are the root causes of the other variables to start our analysis of their causal relations?

Let's think and analyze step by step. Then, provide your final answer (variable names only) within the tags <Answer>...</Answer>, separated by ", ".
\end{tcolorbox}

\begin{tcolorbox}[
    colback=red!5!white,
    colframe=red!75!black,
    title= LeGIT response for Asia dataset,
    breakable,
    listing only,
    listing options={
        breaklines=true,
        basicstyle=\ttfamily\small,
        columns=fullflexible
    }
]
To approach this systematically, let's first echo the knowledge about these variables and their potential relationships, analyze their temporal and causal connections, and then determine the best intervention targets to uncover their causal relations.

\#\#\# Echoing Knowledge of Variables

1. **<smoke>**: Smoking is a well-known risk factor for many respiratory conditions, including bronchitis and lung cancer.

2. **<bronc>**: Bronchitis is an inflammation of the bronchi in the lungs and can be acute or chronic. It can be caused by smoking, among other factors.

3. **<lung>**: Lung cancer is a serious condition that can be caused by factors like smoking and possibly influenced by environmental exposures.

4. **<dysp>**: Dyspnoea, or shortness of breath, can be a symptom of many conditions, including lung cancer, bronchitis, and tuberculosis.

5. **<asia>**: Visiting Asia is included likely due to the higher prevalence of certain diseases, such as tuberculosis, in certain parts of Asia compared to other regions.

6. **<tub>**: Tuberculosis is an infectious disease that primarily affects the lungs but can also affect other parts of the body.

7. **<xray>**: A positive chest x-ray can indicate several conditions, including tuberculosis, lung cancer, and possibly severe bronchitis.

8. **<either>**: This variable indicates the presence of either tuberculosis or lung cancer, which are conditions that can both lead to a positive chest x-ray and symptoms like dyspnoea.

\#\#\# Temporal Analysis

The temporal order of these variables is crucial to understanding their causal relationships. For instance:

- **<smoke>** and **<asia>** are environmental/behavioral exposures that precede health outcomes like **<bronc>**, **<lung>**, and **<tub>**.

- **<bronc>**, **<lung>**, and **<tub>** are health conditions that can lead to symptoms like **<dysp>** and diagnostic findings like **<xray>**.

- **<either>** is a composite variable indicating the presence of **<lung>** or **<tub>**, which directly affects the likelihood of **<xray>** and **<dysp>**.

Given these considerations, the best intervention targets to understand causal relationships would likely be those factors that are upstream in the causal chain and can influence multiple downstream outcomes.

\#\#\# Choosing Best Intervention Targets

1. **<smoke>**: As a primary risk factor for both bronchitis and lung cancer, intervening here can reveal causal links to multiple respiratory conditions.

2. **<asia>**: Given its association with increased risk for tuberculosis, intervening here can help understand the direct impact of geographical exposure on disease risk.

3. **<bronc>**: By intervening on bronchitis, we can explore its direct and indirect effects on symptoms like dyspnoea and diagnostic outcomes like x-rays.

4. **<lung>**: Intervening on lung cancer can clarify its causal links to symptoms and diagnostic findings.

5. **<tub>**: Given tuberculosis's role in respiratory health, interventions here can illuminate its effects on symptoms and diagnostic outcomes.

In summary, focusing on **<smoke>**, **<asia>**, **<bronc>**, **<lung>**, and **<tub>** as intervention targets offers a strategic approach to uncovering the root causes and pathways of respiratory diseases and their symptoms.

<Answer>smoke, asia, bronc, lung, tub</Answer>
\end{tcolorbox}

\begin{tcolorbox}[
    colback=red!5!white,
    colframe=red!75!black,
    title= LeGIT response for Child dataset,
    breakable,
    listing only,
    listing options={
        breaklines=true,
        basicstyle=\ttfamily\small,
        columns=fullflexible
    }
]
To address the problem effectively, we will follow a structured approach to understand the potential causal relationships among the variables. We will start by echoing the knowledge of each variable, analyze their temporal relationships, and then identify the best intervention targets.

\#\#\# Echoing Knowledge of Variables

1. **RUQO2**: Level of oxygen in the right upper quadriceps muscle.

2. **LVHreport**: Report of having left ventricular hypertrophy.

3. **BirthAsphyxia**: Lack of oxygen to the blood during the infant's birth.

4. **DuctFlow**: Blood flow across the ductus arteriosus.

5. **Disease**: Infant methemoglobinemia.

6. **LungParench**: The state of the blood vessels in the lungs.

7. **CO2Report**: A document reporting high levels of CO2 in blood.

8. **LungFlow**: Low blood flow in the lungs.

9. **GruntingReport**: Report of infant grunting.

10. **Age**: Age of infant at disease presentation.

11. **LowerBodyO2**: Level of oxygen in the lower body.

12. **ChestXray**: Having a chest x-ray.

13. **CO2**: Level of CO2 in the body.

14. **HypDistrib**: Low oxygen areas equally distributed around the body.

15. **Grunting**: Grunting in infants.

16. **XrayReport**: Lung excessively filled with blood.

17. **LVH**: Having left ventricular hypertrophy.

18. **CardiacMixing**: Mixing of oxygenated and deoxygenated blood.

19. **HypoxiaInO2**: Hypoxia when breathing oxygen.

20. **Sick**: Presence of an illness.

\#\#\# Temporal Analysis

1. **BirthAsphyxia** is an early event that can lead to multiple downstream effects, such as hypoxia and potential cardiac issues.

2. **CardiacMixing** and **DuctFlow** are congenital or physiological conditions that can affect oxygenation and blood flow.

3. **HypoxiaInO2**, **LungFlow**, and **LowerBodyO2** are likely influenced by earlier conditions like **BirthAsphyxia**, **CardiacMixing**, and **DuctFlow**.

4. **CO2** and **CO2Report** are indicators of respiratory function and can be affected by lung and cardiac issues.

5. **LungParench**, **XrayReport**, and **ChestXray** are diagnostic or descriptive of lung conditions.

6. **Grunting** and **GruntingReport** are symptoms that may arise from respiratory distress.

7. **LVH** and **LVHreport** are indicative of cardiac stress or abnormalities, potentially secondary to other physiological issues.

8. **Disease** (infant methemoglobinemia) could be a result of hypoxic conditions and other metabolic issues.

9. **Sick** is a general indicator of illness and can be influenced by multiple preceding factors.

\#\#\# Choosing Intervention Targets

To identify root causes, we need to target variables that likely influence many others and are early in the causal chain:

1. **BirthAsphyxia**: Early event potentially causing hypoxia, cardiac stress, and other downstream effects.

2. **CardiacMixing**: Affects oxygenation and can lead to hypoxia and other cardiac issues.

3. **DuctFlow**: Influences blood flow and oxygenation, impacting many other variables.

4. **LungFlow**: Directly affects oxygenation and can lead to hypoxia and respiratory distress.

These variables are likely to be root causes or early influencers in the causal network, making them ideal intervention targets to explore causal relationships.

\#\#\# Final Answer

<Answer>BirthAsphyxia, CardiacMixing, DuctFlow, LungFlow</Answer>
\end{tcolorbox}

\begin{tcolorbox}[
    colback=red!5!white,
    colframe=red!75!black,
    title= LeGIT response for Insurance dataset,
    breakable,
    listing only,
    listing options={
        breaklines=true,
        basicstyle=\ttfamily\small,
        columns=fullflexible
    }
]
To analyze the causal relationships among these variables and select the best intervention targets, let's first categorize and understand the potential causal structure among them.

\#\#\# Step 1: Echoing Knowledge of Variables

1. **<ThisCar>**: Costs for the insured car

2. **<Mileage>**: Mileage on the car

3. **<DrivQuality>**: Driving quality

4. **<RiskAversion>**: Being risk-averse

5. **<CarValue>**: Value of the car

6. **<MakeModel>**: Owning a sports car

7. **<ThisCarDam>**: Damage to the car

8. **<Theft>**: Theft occurred in the car

9. **<ILiCost>**: Inspection cost

10. **<Age>**: Age

11. **<OtherCar>**: Involvement with other cars in an accident

12. **<SeniorTrain>**: Received additional driving training

13. **<SocioEcon>**: Socioeconomic status

14. **<VehicleYear>**: Year of vehicle

15. **<HomeBase>**: Neighborhood type

16. **<DrivHist>**: Driving history

17. **<Accident>**: Severity of the accident

18. **<Airbag>**: Car has an airbag

19. **<AntiTheft>**: Car has anti-theft

20. **<ThisCarCost>**: Costs for the insured car

21. **<MedCost>**: Cost of medical treatment

22. **<Antilock>**: Car has anti-lock

23. **<GoodStudent>**: Being a good student driver

24. **<OtherCarCost>**: Cost of the other cars

25. **<Cushioning>**: Quality of cushioning in the car

26. **<RuggedAuto>**: Ruggedness of the car

27. **<PropCost>**: Ratio of the cost for the two cars

28. **<DrivingSkill>**: Driving skill

\#\#\# Step 2: Temporal Analysis of Relations

- **Demographic and Personal Factors**:

  - **<Age>**, **<RiskAversion>**, **<SocioEcon>**, **<GoodStudent>**: These are inherent traits or long-term conditions that can influence driving behavior and decisions.
  
- **Vehicle Characteristics**:

  - **<VehicleYear>**, **<MakeModel>**, **<CarValue>**, **<Airbag>**, **<AntiTheft>**, **<Antilock>**, **<Cushioning>**, **<RuggedAuto>**: These are intrinsic to the car and can influence safety and risk but are typically decided before other events.

- **Driving Experience and Skills**:

  - **<DrivQuality>**, **<DrivingSkill>**, **<DrivHist>**, **<SeniorTrain>**: These can evolve over time and influence accident likelihood.

- **Environmental Factors**:

  - **<HomeBase>**: Affects exposure to theft and accidents.

- **Incident and Cost Factors**:

  - **<ThisCarDam>**, **<Theft>**, **<Accident>**, **<ILiCost>**, **<ThisCarCost>**, **<MedCost>**, **<OtherCarCost>**, **<PropCost>**: These are outcomes or costs associated with incidents.

\#\#\# Step 3: Choosing Intervention Targets

To identify root causes, we should consider intervening on variables that are likely to influence many others and are not themselves primarily effects of other variables. The best interventions would be on variables that are upstream in the causal chain and can potentially change multiple downstream outcomes.

1. **<RiskAversion>**: Can influence driving behavior, choice of car, and safety measures.

2. **<SocioEcon>**: Affects car choice, home base, and potentially driving habits.

3. **<MakeModel>**: Influences car value, safety features, and possibly driving style.

4. **<DrivingSkill>**: Directly impacts driving quality, accident likelihood, and associated costs.

These variables are chosen because they are fundamental characteristics or decisions that precede and potentially influence many other variables in the system.

\#\#\# Final Answer

<Answer>RiskAversion, SocioEcon, MakeModel, DrivingSkill</Answer>
\end{tcolorbox}

\begin{tcolorbox}[
    colback=red!5!white,
    colframe=red!75!black,
    title= LeGIT response for Insurance dataset,
    breakable,
    listing only,
    listing options={
        breaklines=true,
        basicstyle=\ttfamily\small,
        columns=fullflexible
    }
]
To determine the best intervention targets, we need to consider the relationships and possible temporal orders among the variables. Here's a step-by-step analysis:

\#\#\# Step 1: Understanding the Variables

1. **Respiratory and Ventilation Variables:**

   - **<VENTMACH>**: Intensity level of a breathing machine.
   
   - **<MINVOLSET>**: Time using a breathing machine.
   
   - **<VENTLUNG>**: Lung ventilation.
   
   - **<PRESS>**: Breathing pressure.
   
   - **<MINVOL>**: Minute volume; gas inhaled/exhaled per minute.
   
   - **<VENTTUBE>**: Presence of a breathing tube.
   
   - **<VENTALV>**: Gas exchange in alveoli.
   
   - **<DISCONNECT>**: Disconnection.
   
   - **<EXPCO2>**: Expelled CO2.
   
   - **<ARTCO2>**: Arterial carbon dioxide.
   
   - **<FIO2>**: Oxygen concentration in inspired gas.

2. **Cardiovascular Variables:**

   - **<HREKG>**: Heart rate on EKG.
   
   - **<PCWP>**: Pulmonary capillary wedge pressure.
   
   - **<HRBP>**: Ratio of heart rate and blood pressure.
   
   - **<BP>**: Blood pressure.
   
   - **<CO>**: Cardiac output.
   
   - **<HR>**: Heart rate.
   
   - **<TPR>**: Total peripheral resistance.
   
   - **<STROKEVOLUME>**: Blood volume per heart contraction.
   
   - **<LVEDVOLUME>**: Blood in left ventricle before contraction.
   
   - **<PAP>**: Pulmonary artery pressure.
   
   - **<CVP>**: Central venous pressure.
   
   - **<SAO2>**: Oxygen saturation of arterial blood.
   
   - **<PVSAT>**: Oxygen saturation in pulmonary artery.
   
   - **<HRSAT>**: Hemoglobin oxygen saturation.
   
   - **<CATECHOL>**: Adrenal hormones.
   
3. **Medical Conditions and Procedures:**

   - **<ANAPHYLAXIS>**: Severe allergic reaction.
   
   - **<INTUBATION>**: Insertion of a breathing tube.
   
   - **<INSUFFANESTH>**: Insufficient anesthesia.
   
   - **<PULMEMBOLUS>**: Pulmonary embolism.
   
   - **<HYPOVOLEMIA>**: Fluid loss condition.
   
   - **<LVFAILURE>**: Left ventricle failure.
   
   - **<SHUNT>**: Drainage tube for cerebrospinal fluid.
   
   - **<KINKEDTUBE>**: Kinked chest tube.

4. **Errors and Anomalies:**

   - **<ERRLOWOUTPUT>**: Low output error.
   
   - **<ERRCAUTER>**: Cautery error.

5. **Other:**
   - **<HISTORY>**: Previous medical history.

\#\#\# Step 2: Temporal Analysis

- **<INTUBATION>** and **<VENTMACH>** are likely initial interventions that can influence many respiratory and cardiovascular variables.

- **<ANAPHYLAXIS>**, **<PULMEMBOLUS>**, and **<HYPOVOLEMIA>** are acute conditions that can cause significant changes in cardiovascular and respiratory variables.

- **<INSUFFANESTH>** can affect heart rate and blood pressure due to inadequate sedation.

- **<VENTLUNG>**, **<VENTALV>**, and **<MINVOL>** are outcomes of ventilation interventions.

\#\#\# Step 3: Choosing the Best Intervention Targets

To determine causality effectively, we want to target variables that are likely root causes or significant influencers:

1. **<INTUBATION>**: This is a direct intervention that can affect many respiratory variables.

2. **<VENTMACH>**: Controls the intensity of mechanical ventilation, influencing respiratory and possibly cardiovascular variables.

3. **<ANAPHYLAXIS>**: A condition that can have widespread effects on cardiovascular and respiratory systems.

4. **<HYPOVOLEMIA>**: Can significantly impact cardiovascular variables like blood pressure and cardiac output.

These interventions can provide insights into the causal relationships within the system.

\#\#\# Final Answer

<Answer>INTUBATION, VENTMACH, ANAPHYLAXIS, HYPOVOLEMIA</Answer>
\end{tcolorbox}


\end{document}